\documentclass[10pt,twocolumn,letterpaper]{article}

\usepackage{cvpr}              % To produce the CAMERA-READY version
\definecolor{cvprblue}{rgb}{0.21,0.49,0.74}
\usepackage[pagebackref,breaklinks,colorlinks,allcolors=cvprblue]{hyperref}
\usepackage{algorithm}
\usepackage{amsthm}
\usepackage{multirow}  
\usepackage{colortbl}
\usepackage{wrapfig}
\usepackage{amssymb}
\usepackage{xspace}

\usepackage[noend]{algpseudocode} % from algorithmicx

\newcommand{\method}{\emph{Spectrum}\xspace}

\def\paperID{} % *** Enter the Paper ID here
\def\confName{CVPR}
\def\confYear{2026}

\title{Adaptive Spectral Feature Forecasting for Diffusion Sampling Acceleration}

\makeatletter
\def\@fnsymbol#1{\ensuremath{\ifcase#1\or \dagger\or \ddagger\or
   \mathsection\or \mathparagraph\or \|\or **\or \dagger\dagger
   \or \ddagger\ddagger \else\@ctrerr\fi}}
    \makeatother

\author{Jiaqi Han$^{1\dagger}$ \hspace{0.8em} Juntong Shi$^{1\dagger}$ \hspace{0.8em} Puheng Li$^1$ \hspace{0.8em} Haotian Ye$^1$ \hspace{0.8em} Qiushan Guo$^2$ \hspace{0.8em} Stefano Ermon$^1$\\
$^1$Stanford University \hspace{2em} $^2$ByteDance\\
\url{https://hanjq17.github.io/Spectrum} \\
}

\usepackage{amsmath,amsfonts,bm}

\def\eqref#1{equation~\ref{#1}}
\def\floor#1{\lfloor #1 \rfloor}
\def\1{\bm{1}}

\def\rd{{\textnormal{d}}}

\def\rvh{{\mathbf{h}}}

\def\rvx{{\mathbf{x}}}

\def\rmA{{\mathbf{A}}}

\def\rmC{{\mathbf{C}}}

\def\rmE{{\mathbf{E}}}

\def\rmH{{\mathbf{H}}}
\def\rmI{{\mathbf{I}}}

\DeclareMathAlphabet{\mathsfit}{\encodingdefault}{\sfdefault}{m}{sl}
\SetMathAlphabet{\mathsfit}{bold}{\encodingdefault}{\sfdefault}{bx}{n}

\def\gA{{\mathcal{A}}}

\def\gN{{\mathcal{N}}}
\def\gO{{\mathcal{O}}}

\def\gR{{\mathcal{R}}}

\def\gW{{\mathcal{W}}}
\def\gX{{\mathcal{X}}}

\def\sC{{\mathbb{C}}}

\def\sN{{\mathbb{N}}}

\def\sR{{\mathbb{R}}}

\def\sT{{\mathbb{T}}}
\def\sU{{\mathbb{U}}}
\def\sV{{\mathbb{V}}}

\DeclareMathOperator*{\argmin}{arg\,min}

\newtheorem{theorem}{Theorem}[section]

\begin{document}
\twocolumn[{
\maketitle
\centering
\vspace{-0.5cm}
\captionsetup{type=figure}
\includegraphics[width=\linewidth]{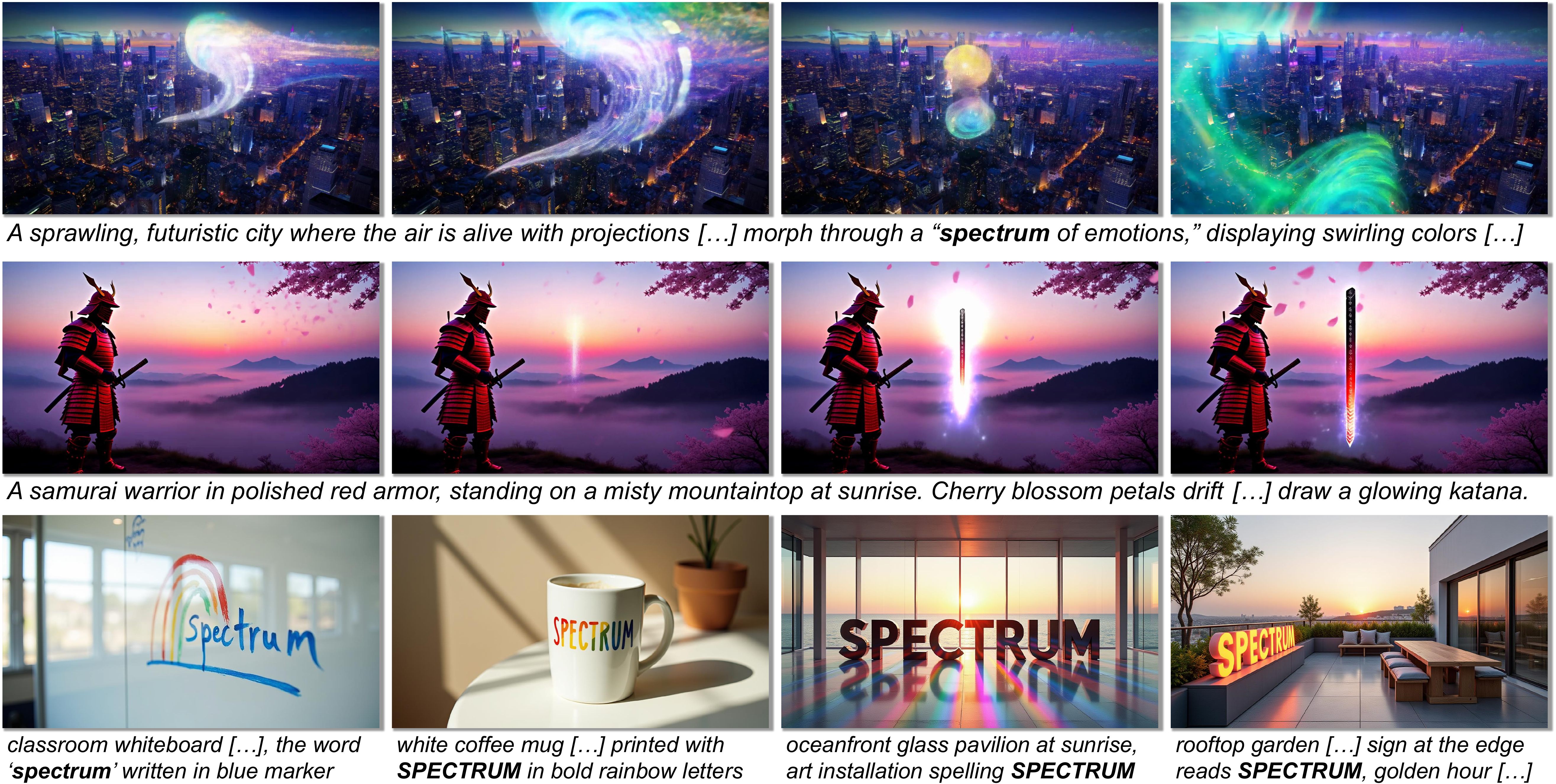}
\vskip -0.05in
\caption{Text-to-video (top) and text-to-image (bottom) generated samples using our~\method on HunyuanVideo and FLUX.1, respectively. Samples were generated using only \textbf{14 network evaluations}, leading to a significant speedup of 3.5$\times$ without quality degradation.}
\label{fig:show_muscle}
\vskip 0.2in
% \vspace{0.7cm}
}]

\def\thefootnote{$\dagger$}\footnotetext{Equal. Correspondence to~\url{jiaqihan@stanford.edu}.} 
\def\thefootnote{\arabic{footnote}}

\begin{abstract}
Diffusion models have become the dominant tool for high-fidelity image and video generation, yet are critically bottlenecked by their inference speed due to the numerous iterative passes of Diffusion Transformers. To reduce the exhaustive compute, recent works resort to the feature caching and reusing scheme that skips network evaluations at selected diffusion steps by using cached features in previous steps. However, their preliminary design solely relies on local approximation, causing errors to grow rapidly with large skips and leading to degraded sample quality at high speedups. In this work, we propose spectral diffusion feature forecaster (\method), a training-free approach that enables global, long-range feature reuse with tightly controlled error. In particular, we view the latent features of the denoiser as functions over time and approximate them with Chebyshev polynomials. Specifically, we fit the coefficient for each basis via ridge regression, which is then leveraged to forecast features at multiple future diffusion steps. We theoretically reveal that our approach admits more favorable long-horizon behavior and yields an error bound that does not compound with the step size.
Extensive experiments on various state-of-the-art image and video diffusion models consistently verify the superiority of our approach. Notably, we achieve up to 4.79$\times$ speedup on FLUX.1 and 4.67$\times$ speedup on Wan2.1-14B, while maintaining much higher sample quality compared with the baselines.
\end{abstract}

\section{Introduction}
\label{sec:intro}

Diffusion models have become the preeminent tool for high-fidelity image and video synthesis~\cite{rombach2022high,song2019generative,song2021scorebased,sohl2015deep,ho2020denoising,esser2024scaling,yang2023diffusionm}. Their success, however, comes with a steep computational price: their sampling requires dozens to hundreds of evaluations of a large denoiser, which, for most state-of-the-art models, is typically a Diffusion Transformer~\cite{peebles2023scalablediffusionmodelstransformers,liu2025regionadaptivesamplingdiffusiontransformers} with deep stacks of attention blocks. The dominant cost of inference poses practical limits for interactive applications that call for real-time model responses.

To address the challenge, various works have been proposed to reduce the sampling latency of diffusion models. Amongst them, feature caching-based approaches~\cite{lvFasterCacheTrainingFreeVideo2025,zou2024accelerating,zou2024DuCa,liu2024timestep,liu2025taylor} have stood out and shown great promise in remarkably reducing the sampling time while maintaining desirable sample quality, all without additional training. Concretely, these methods cache latent features produced by certain blocks at selected timesteps and use them to synthesize features at future timesteps via light-weight predictors, thereby avoiding expensive denoiser evaluations. In particular, the naïve copy strategy directly reuses the most recent cached feature~\cite{liu2024timestep,selvaraju2024fora}, while the recent work~\cite{liu2025taylor} performs discrete Taylor expansion using a few nearest cached points.

Despite their promise, through our error analysis and empirical verification, we find these \emph{local} predictors incur approximation errors that compound rapidly as the forecasting horizon grows, leading to severe quality degradation at high speedup ratios. Furthermore, the local predictors, such as Taylor expansion, struggle to capture the long-range temporal dynamics of features along the sampling process, and are thus prone to producing images with inaccurate semantics. These limitations severely hinder their practical utility, particularly in high-speedup scenarios.

% \paragraph{From local to spectral forecasting.}
% \looseness=-1
To this end, we move beyond local approximations and instead construct the forecaster in the \emph{spectral} domain. Our key insight is to view each feature channel of the denoiser output along the diffusion trajectory as a function over time and approximate it with a compact set of global, orthonormal bases in function space. Notably, we leverage \emph{Chebyshev polynomials}~\cite{rivlin1974chebyshev} to fit the functions via a ridge regression objective, motivated by their well-conditioned recurrences and favorable error behavior at long horizons~\cite{mason2002chebyshev}. The fitted coefficients are then employed to predict features at future steps, enabling large skips of actual network evaluations. The framework, which we coin \emph{Spectral Diffusion Feature Forecaster} (\method), not only yields more tightly bounded error but also injects crucial long-range spectral semantics to the generated sample, substantially reducing the error and enhancing sample quality. With the benefits of such design,~\method produces image and video samples using only 14 network evaluations without quality degradation, as demonstrated in Fig.~\ref{fig:show_muscle}.

In summary, we make the following major contributions. \textbf{1.} We introduce spectral diffusion feature forecasting (\method), the first feature caching-based diffusion acceleration approach that operates in the spectral domain. Critically, we leverage Chebyshev polynomials as the basis set due to its desirable properties for error reduction. \textbf{2.} We perform thorough error analysis for~\method and compare it against previous local predictors, demonstrating tighter error bound at large time horizons. Notably,~\method admits an error bound that does not scale with the skip horizon. \textbf{3.} We propose a tractable and efficient method to fit the coefficients of the Chebyshev bases with negligible compute and memory overhead. \textbf{4.} We conduct extensive experiments on a suite of state-of-the-art image and video diffusion models. Results consistently verify that~\method yields remarkable speedups while maintaining much higher sample quality compared with previous feature caching approaches. Notably,~\method obtains up to 4.79$\times$ speedup on FLUX.1 and 3.40 to 4.67$\times$ speedup on Wan2.1-14B, all without notable quality degradation. These evidence favors our~\method as a solid step towards streamlined, real-time diffusion generation.

\section{Related Work}
\label{sec:rw}

Rich efforts have been taken to reduce the sampling latency of diffusion models, which we discuss as follows.

\noindent\textbf{Diffusion sampling acceleration via distillation.} This thread of works aims to distill multi-step teacher models into few-step students via additional training. Specifically, consistency models~\cite{song2023consistency,song2024improved} and subsequent efforts~\cite{kim2023consistency,lu2025simplifying} enforce consistency along the diffusion trajectory, while DMD~\cite{yin2024stepdistillation,yin2024improved} and related works~\cite{salimans2022progressive,salimans2024multistep} distill the student model by aligning with teacher score. Recent endeavors~\cite{geng2025mean,sabour2025align,liu2023flow} adopt these techniques to flow matching models~\cite{liu2023flow,lipman2023flow} and have demonstrated success. 
While effective, these approaches require extra training and is prone to decreased diversity due to distillation, while our method is purely \emph{training-free} and leaves the base denoiser intact.

\noindent\textbf{Diffusion sampling acceleration with ODE solvers.}
Another family of work decreases the number of integration steps by leveraging more accurate ODE solvers of the probability flow ODE~\cite{song2021scorebased}. Pioneered by DDIM~\cite{song2020denoising}, various works, \emph{e.g.}, DPMSolver~\cite{lu2022dpm,lu2022dpm++,zheng2023dpmsolvervF}, have introduced high-order solvers that simulate the reverse ODE with less error accumulation. Yet, when the number of sampling steps is extremely limited, they still deviate significantly from the original trajectory, leading to quality degradation. Our~\method is agnostic to the choice of solver and remains close to the reference trajectory by replacing unnecessary network evaluations with inexpensive forecasts.

\noindent\textbf{Diffusion acceleration by reusing cached features.}
Most related to our present approach is the diffusion feature caching methods that avoid executing the denoiser at selected timesteps by using the cached features from historical activations. The simplest approach naively reuses the nearest cached latents~\cite{selvaraju2024fora,chen2024delta-dit}, which over-simplifies the feature correlation, leading to drops in sample quality. ToCa~\cite{zou2024accelerating} and subsequent works~\cite{Zheng2025Compute,zou2024DuCa,Liu2025SpeCa} dynamically cache attention features.
TeaCache~\cite{liu2024timestep} estimates time-dependent input-output difference to actively decide when to cache, while recent works~\cite{lvFasterCacheTrainingFreeVideo2025,sun2025unicpunifiedcachingpruning,liu2025regionadaptivesamplingdiffusiontransformers,guan2025forecasting} develop various dynamic caching strategies.
Taylor-style forecasters improve upon this by discrete finite-difference expansions over a short temporal horizon~\cite{liu2025taylor}. However, their locality leads to large discretization error, making them less favorable in high speedup scenarios.
Differently,~\method shifts from \emph{local} time-domain approximations to \emph{global} spectral forecasting, yielding better long-range behavior and more accurate simulation along the sampling trajectory.

\noindent\textbf{Other techniques.}
A complementary stream reduces the cost per denoiser call through quantization~\cite{li2023q,kim2025ditto,shang2023post}, network pruning~\cite{structural_pruning_diffusion,zhu2024dipgo}, token reduction~\cite{zhang2024tokenpruningcachingbetter,cheng2025catpruningclusterawaretoken,saghatchian2025cached}, or kernel-level optimizations~\cite{Dao2022FlashAttentionFA}. Parallel sampling has also been explored that partitions and solves the trajectory with multiple cores~\cite{shih2023parallel,tang2024acceleratingparallelsamplingdiffusion,han2025chords}. These techniques are indeed orthogonal to our~\method, which can be potentially combined together to further improve the wall-clock performance.

\section{Method}
\label{sec:method}

\subsection{Feature Forecasting for Diffusion Acceleration}
\label{sec:feature-forecasting}
\textbf{Diffusion model and diffusion sampling.} Given a data point $\rvx\in\gX=\sR^{F}$ sampled from the data distribution, diffusion models feature a forward process that progressively perturbs the input towards unit Gaussian noise $\bm\epsilon\sim\gN(\bm{0},\rmI)$.
A neural network 
$
\bm\epsilon_\theta : \mathcal X \times [0,1] \mapsto \mathcal X
$
is optimized to approach the score function $\nabla_\rvx\log p_t(\rvx)$ through, \emph{e.g.}, denoising score matching~\cite{song2021scorebased,song2019generative}, where $p_t$ is the marginal at time $t$.  Critically, the sampling of diffusion models is grounded on the reverse process, typically parameterized by the probability-flow ODE~\cite{song2021scorebased}:
\begin{align}
\label{eq:pf-ode}
    \rd\rvx=\left(f(t)\rvx-\frac{1}{2}g^2(t)\bm \epsilon_\theta(\rvx, t)\right)\rd t,
\end{align}
where $f(t)$ and $g(t)$ are scalar functions specified by the noise schedule. Various types of ODE solvers~\cite{karras2022elucidating,song2020denoising,lu2022dpm} can be readily applied to simulate the ODE via, \emph{e.g.}, discretizing uniformly over $N$ timesteps $\sT=\{t_i\}_{i=1}^N$ with $t_i=(i-1)\delta_t$, where $\delta_t=\frac{1}{N}$ is the interval between two adjacent timesteps. For instance, in the case of Euler solver, each step can be formulated as the following iteration:
\begin{align}
\label{eq:ode-solve-step}
    \rvx_{t_{i+1}} = \mathrm{Solve}\left(\rvx_{t_i}, \bm\epsilon_\theta, t_i,t_{i+1}  \right).
\end{align}
To obtain the clean data at $t=1$, Eq.~\ref{eq:ode-solve-step} will be iterated $N$ times, leading to a total of $N$ function calls to the denoiser $\bm\epsilon_\theta$. Empirically, for the state-of-the-art diffusion models, the denoiser is usually parameterized by Diffusion Transformers~\cite{peebles2023scalablediffusionmodelstransformers} that comprise a sequential stack of dozens of attention blocks, making the entire sampling process compute-demanding and time-consuming.

\noindent\textbf{Feature caching and reusing.} To tackle the challenge, a line of works propose to cache the output features of the attention blocks at previous timesteps and reuse them to skip the actual forward at future timesteps. Formally, we unify this class of algorithms as $\gA=(\sU,\sV,f)$ where $\sU\subseteq\sT$ denotes the set of timesteps that requires the expensive forward pass of the denoiser $\bm\epsilon_\theta$, and $f$ is the forecaster whose output is employed as the surrogate of the network output for timesteps in $\sV=\sT\setminus\sU$. Specifically, at certain timestep $t_j\in\sV$, the forecaster $f$ predicts the feature by using the cached features at the same block\footnote{We omit the block index $l$ for conciseness.}:
\begin{align}
\label{eq:cache-def}
    {\rvh}_{t_j} = f(\sC_{t_j}   ),\quad \text{s.t.} \ \  \sC_{t_j}=\{ (\rvh_{t_k}, t_k) : t_k < t_j,t_k\in\sU  \},
\end{align}
where $\sC_{t_j}$ is the \emph{feature cache} at the current timestep $t_j$, consisting of the features $\rvh_{t_k}\in\sR^F$ collected from actual passes of $\bm\epsilon_\theta$ at timesteps $t_k\in\sU$. The theoretical speedup measured by the ratio of the number of network evaluations will be $\eta=(\|\sU\|+\|\sV\|)/\|\sU\|=\|\sT\|/\|\sU\|$.

\noindent\textbf{Existing designs of $\gA$.} Empirically, due to the strong temporal proximity of the latent features between adjacent timesteps, several designs of $f$ have been found highly effective, including:
\begin{itemize}
    \item Naive reusing~\cite{selvaraju2024fora,liu2024timestep}, which directly copies the cached feature at the nearest timestep $t_k$ in $\sC_{t_j}$, \emph{i.e.}, $f(\sC_{t_j})=\rvh_{t_k}$ where $t_k=\max_t \{(\rvh,t)\in \sC_{t_j}  \}$.

    \item TaylorSeer~\cite{liu2025taylor}, which leverages a discrete Taylor expansion that predicts the feature based on $P+1$ nearest cached timesteps:
    \begin{align}
    \label{eq:taylor}
       \rvh_{t_j}^{\text{Taylor}} = \rvh_{t_k}+\sum_{p=1}^P \frac{\Delta^p\rvh_{t_k}}{p!}(\frac{j-k}{\eta})^p,
    \end{align}
    where $\Delta^p$ is the $p$-th order forward differential operator, which is defined as 
    \[
    \Delta^p \rvh_{t_k}=\sum_{i=0}^p(-1)^{p-i}\binom{p}{i} \rvh_{t_{k - i\eta}}, \quad p\leq \frac{k}{\eta}.
    \]
\end{itemize}
Notably, naive reusing can be viewed as the special case of TaylorSeer when the maximum order $P=0$.

For the timesteps $\sU$ to perform an actual pass, the \emph{fixed interval} strategy is commonly employed, \emph{i.e.}, $\sU=\{ (j-1)\eta\delta_t \}_{j=1}^{\floor{N/\eta}}$ for speedup ratio $\eta$ and unit timestep $\delta_t$.

% Define diffusion sampling procedure and notations.
% Define feature forecasting formulation.
\subsection{From Local to Global: Bounding Forecasting Errors with Chebyshev Polynomials}
\noindent\textbf{Pitfalls of the Taylor forecaster.} While demonstrated to be effective, from Eq.~\ref{eq:taylor} it is evident that the prediction of the Taylor forecaster is grounded on discrete Taylor expansion using several most adjacent timesteps. Such a nature is prone to introducing a significantly growing approximation error when the forecasting step is large. Our observation is motivated by the following error analysis on the Taylor forecaster, as stated in the theorem below:
\begin{theorem}[Worst-case error for local order-$P$ Taylor]
\label{thm:taylor-minimax-lb}
    Fix an expansion point $\tau_k\in[0,1]$ and a target $\tau_j=\tau_k+ (j - k) \delta_t$.
Consider the smoothness class
\small{
\[
\mathcal{F}_{P+1}(L)\;=\;\bigl\{f\in C^{P+1}([0,1])\;:\;\|f^{(P+1)}\|_\infty \le L\bigr\},\quad L>0.
\]}
Let $T_P[f](\tau_j)$ denote the \emph{ideal} order-$P$ Taylor predictor of $f(\tau_j)$ centered at $\tau_k$ using the exact derivatives $f^{(p)}(\tau_k)$, $p\le P$.
Then
\[
\sup_{f\in \mathcal{F}_{P+1}(L)}\; \big|\, f(\tau_j)-T_P[f](\tau_j)\,\big| \;=\; \frac{L}{(P+1)!}\,\left((j-k)\delta_t\right)^{P+1}.
\]
\end{theorem}
With the worst-case error, we locate the root cause that is detrimental to the forecasting accuracy of the Taylor method particularly when the speedup ratio $\eta$ is high: the error will be dramatically amplified as the forecasting step size $\tau_j-\tau_k$ increases, which leads to degraded sample quality for high speedups. This is due to the fact that the Taylor forecaster heavily relies on the features of the few adjacent timesteps as required by the local Taylor expansion, which is also highly ineffective in capturing the global, long-range spectral features produced along the entire sampling trajectory. Fig.~\ref{fig:video-qualitative-compare} illustrates such an effect that the Taylor forecaster exaggerates the local sharp details while failing to capture the spectral features that are important to the overall semantics.

 Realizing these clear limitations of the Taylor forecaster, we introduce \emph{spectral feature forecasters}, which address the pitfalls above by switching to the spectral domain that models the sampling trajectory \emph{globally}. At its core lies Chebyshev polynomials, a powerful tool that serves as a set of spectral bases for the temporal functions of our interest.

\noindent\textbf{Chebyshev polynomials}. Chebyshev polynomials of the first kind~\cite{mason2002chebyshev} are defined by the following recurrence,
\begin{align}
    T_m(\tau)=2\tau T_{m-1}(\tau)-T_{m-2}(\tau),
\end{align}
with the base $T_0(\tau)=1$ and $T_1(\tau)=\tau$ for $\tau\in[-1,1]$. An important property of the Chebyshev polynomials is that they form a set of orthonormal basis for the function space, such that any function can be represented as a weighted sum of the Chebyshev polynomials:
\begin{theorem}[Universality of Chebyshev Polynomials]
    Let $f:[-1,1]\!\to\!\sR$ and let $E_M(f):=\inf_{\deg(p)\le M}\|f-p\|_\infty$.
    If $f$ extends analytically to the Bernstein ellipse $E_\rho$ with parameter $\rho>1$ and $\|f\|_{E_\rho}\le B$, then for the truncated Chebyshev series $p_M$ of degree $M$,
\begin{equation}
\label{eq:bernstein}
\|f-p_M\|_\infty \ \le\ \frac{2B}{\rho-1}\,\rho^{-M}.
\end{equation}
\label{theo:universal}
\end{theorem}
Theorem~\ref{theo:universal} states a critical property that lies in the heart of our spectral forecaster: Chebyshev polynomial approximates the target function uniformly regardless of the forecasting step size. Specifically, the approximation is bounded by the degree $M$ as opposed to the step size, making our spectral forecaster much more robust to large step sizes enforced in high speedup scenarios.

\subsection{Spectral Diffusion Feature Forecasting}
With Chebyshev polynomials introduced, we are now ready to present the algorithm framework of our Spectral Feature Forecaster (\method) for accelerating diffusion sampling. Specifically, we propose to view each channel of the feature vector\footnote{Here feature $\rvh_t\in\sR^F$ refers to the \emph{fully flattened} network outputs.} as a function that evolves through time $t$, \emph{i.e.},
\begin{align}
    \rvh_{t}=[h_1(t),\cdots,h_F(t)]\in\sR^{F},
\end{align}
where $h_i(t):\sR\mapsto\sR$ is the function for the feature at channel $i$. For each $h_i(t)$, we leverage a set of Chebyshev polynomials up to degree $M$ to approximate it, yielding
\begin{align}
\label{eq:basic-form}
    h_i(t)=\sum_{m=0}^M c_{m,i} T_m(\tau),\quad\text{s.t.}\ \ \tau=g(t),
\end{align}
where $g(t):[0,1]\mapsto [-1,1]=2t-1$ is the timestep projection that maps diffusion timestep to $[-1,1]$ to match the support of Chebyshev polynomials, and $c_{m,i}\in\sR$ is the coefficient for the $m$-th polynomial. Since we have no knowledge beforehand of the coefficients, we propose to fit them in an online manner during the sampling process. 

\noindent\textbf{Online coefficient fitting.} For any timestep $t_j$, according to the protocol in Eq.~\ref{eq:cache-def} we have defined for $\gA$, we have maintained the feature cache $\sC_{t_j}$ which consists of the features and timesteps collected from previous actual network evaluations. In order to perform the fitting, we first construct the basis row vector as
\begin{align}
\label{eq:basis-row}
    \bm\phi(\tau_k)=[T_0(\tau_k),T_1(\tau_k),\cdots,T_M(\tau_k)]\in\sR^{M+1},
\end{align}
which evaluates the Chebyshev polynomials at $\tau_k=g(t_k)$ for all $t_k$ recorded in $\sC_{t_j}$. The design matrix $\bm\Phi_{t_j}$ is then instantiated as a stack of basis row vectors along time axis:
\begin{align}
\label{eq:design_matrix}
    \bm\Phi_{t_j}=[\bm\phi(g(t_k))]_{(\rvh_{t_k},t_k)\in\sC_{t_j}} \in\sR^{K\times (M+1)},
\end{align}
where $K=|\sC_{t_j}|$ is the total number of currently cached timesteps. Similarly, the feature matrix is built by piling up the cached features over time:
\begin{align}
    \rmH_{t_j}=[\rvh_{t_k}]_{(\rvh_{t_k},t_k)\in\sC_{t_j}}\in\sR^{K\times F}.
\end{align}
To obtain the coefficient matrix $\rmC_{t_j}\in\sR^{(M+1)\times F}$ that comprises the coefficients $c_{m,i}=\rmC_{t_j}[m,i]$ used in Eq.~\ref{eq:basic-form}, we solve the following ridge regression objective:
\begin{align}
\label{eq:ridge-problem}
    \rmC_{t_j}=\argmin_{\rmC} \| \bm\Phi_{t_j}\rmC - \rmH_{t_j}   \|_F^2 + \lambda \|\rmC\|_F^2,
\end{align}
where $\lambda\in\sR$ is the regularization strength. It is worth noting that the regularization term is critical in terms of enhancing stability and relieving overfitting, as we will demonstrate in our experiment. The problem in Eq.~\ref{eq:ridge-problem} has the following close-form solution:
\begin{align}
\label{eq:ridge-solve}
    \rmC_{t_j}={\underbrace{(\bm\Phi_{t_j}^\top\bm\Phi_{t_j}+\lambda\rmI)}_{\in\sR^{(M+1)\times(M+1)}}}^{-1}\bm\Phi_{t_j}^\top\rmH_{t_j},
\end{align}
which can be efficiently solved via Cholesky decomposition. Note that the matrix inversion term is of shape $(M+1)\times (M+1)$, which is negligible when the number of basis $M$ remains small, compared with the feature dimension $F$, which is the major compute bound of the solve in Eq.~\ref{eq:ridge-solve}.
The basis fitting for $\rmC_{t_j}$ can be naturally conducted in an online manner. Whenever a new latent feature is produced from $\bm\epsilon_\theta$, we update $\rmC_{t_j}$ by solving Eq.~\ref{eq:ridge-solve} and cache the updated coefficients, which will be leveraged to forecast the features at future timesteps.

\begin{algorithm}[t!]
\small
\caption{Adaptive Spectral Feature Forecaster}
\label{alg:ourmethod}
\textbf{Input:} Diffusion step $N$, timesteps $\{t_j \}_{j=1}^N$, denoiser $\bm\epsilon_\theta$, initial latent $\rvx_0$, diffusion solver $\mathrm{Solve}(\cdot)$.
\begin{algorithmic}[1]
\For {$j = 1,\cdots, N$}
\If {$t_j\in \sU$} \Comment{\textcolor{gray}{If using actual forward pass}} 
\State {$\rvh_{t_j}, \bm\epsilon_{t_j}\leftarrow \bm\epsilon_\theta(\rvx_{t_{j-1}}, t_{j-1})$ \ \ \ \ \ \ \Comment{\textcolor{gray}{Actual forward pass}}}
\State {$\sC_{t_j}\leftarrow \sC_{t_{j-1}}\cup \{ (\rvh_{t_j}, t_j) \}$ \Comment{\textcolor{gray}{Update feature cache}}}
\State {$\rmC_{t_j}\leftarrow\mathrm{Fit}(\sC_{t_j})$ \Comment{\textcolor{gray}{Eq.~\ref{eq:design_matrix}-\ref{eq:ridge-solve}}}}
\Else \Comment{\textcolor{gray}{If using our forecaster}}
\State {$\rmC_{t_j},\sC_{t_j}\leftarrow\rmC_{t_{j-1}},\sC_{t_{j-1}}$ \Comment{\textcolor{gray}{Maintain the cache}}}
\State {$\rvh_{t_j}\leftarrow\mathrm{Forecast}(\rmC_{t_j}, t_j)$ \Comment{\textcolor{gray}{Eq.~\ref{eq:pred}}}}
\State {$\bm\epsilon_{t_j}\leftarrow \tilde{\bm\epsilon} (\rvh_{t_j})$ \Comment{\textcolor{gray}{Obtain the score}}}
\EndIf
\State{$\rvx_{t_j}\leftarrow\mathrm{Solve}(\rvx_{t_{j-1}}, \bm\epsilon_{t_j}, t_{j-1}, t_j)$} \Comment{\textcolor{gray}{Solver step}}
\EndFor
\end{algorithmic}
% \vskip -0.1in
% \vspace{-5pt}
\end{algorithm}
% \vskip -0.1in

\noindent\textbf{Feature forecasting with fitted coefficients.} For the timesteps $t_j\in\sV$, the algorithm skips the network pass of $\bm\epsilon_\theta$ and inherits the cached coefficients $\rmC_{t_j}=\rmC_{t_{j-1}}$ to forecast the feature value at current timestep, following Eq.~\ref{eq:basic-form}:
\begin{align}
\label{eq:pred}
    \rvh_{t_j} = \bm\phi(g(t_j))\rmC_{t_j}.
\end{align}
The predicted $\rvh_{t_j}$ is then directly ensembled throughout the network to produce the estimated score $\bm\epsilon_{t_j}$ for the current step without any expensive full pass over the network.

\noindent\textbf{Overall procedure.} As depicted in Alg.~\ref{alg:ourmethod}, our~\method operates in an online fitting-then-forecasting manner. Specifically, for each sampling step $j$ looping through 1 to $N$, if $t_j\in\sU$, we perform one actual network pass through the denoiser $\bm\epsilon_\theta$ to obtain the latent features $\rvh_{t_j}$ as well as the output score $\bm\epsilon_{t_j}$. The features $\rvh_{t_j}$ will be cached and the Chebyshev coefficients $\rmC_{t_j}$ will be updated correspondingly by solving Eq.~\ref{eq:ridge-solve}. Otherwise, for $t_j\in\sV$, we spare one network evaluation by directly predicting the features via Eq.~\ref{eq:pred} using the previously computed coefficients.

\subsection{Analysis and Extensions}
\label{sec:analysis_and_extensions}
In this subsection, we provide more detailed discussions on~\method, including time and memory complexity analysis. We further propose several empirical extensions of~\method that further enhances its practical performance.

\noindent\textbf{Error analysis.} Theorem~\ref{theo:our-error} below shows an error bound of \method, which is an extention of Theorem~\ref{theo:universal}. Notably, it does not depend on the step size $\tau_j - \tau_k$ as in the Taylor forecaster, which is highly beneficial.
\begin{theorem}[Error Bound of \method]
\label{theo:our-error}
Using the notation in Sec.~\ref{sec:method}, fix a channel $i$ and write
$f(\tau) = h_i(t)$ with $\tau = g(t)\in[-1,1]$.
Assume $f$ extends analytically to the Bernstein ellipse $E_\rho$ with $\rho>1$
and $\|f\|_{E_\rho} \le B$.
Let $p_M$ be the degree-$M$ Chebyshev truncation of $f$ and $\varepsilon_M := \|f - p_M\|_\infty \le \frac{2B}{\rho - 1}\,\rho^{-M}$ as in Theorem~\ref{theo:universal}. For a forecast step $t_j$ with cache $\sC_{t_j}$, let $\widehat h_i(t_j)$ be the \method\ predictor at time $t_j$ obtained
from Eq.~\ref{eq:pred} using the fitted coefficients $\rmC_{t_j}$
from Eq.~\ref{eq:ridge-solve}. Then for any forecast time $t_j$ with $\tau_j = g(t_j)$,
\small{
$$
\begin{aligned}
& \bigl|\,h_i(t_j) - \widehat h_i(t_j)\,\bigr| \\
\;\le\;
&\varepsilon_M
\left(
1 + \frac{(M+1)K}{\sigma_{\min}(\bm\Phi_{t_j})^2 + \lambda}
\right)
+
\frac{\lambda\,\sqrt{M+1}}{\sigma_{\min}(\bm\Phi_{t_j})^2 + \lambda}
\cdot
\frac{2B}{\sqrt{1 - \rho^{-2}}},
\end{aligned}
$$}
where $\sigma_{\min}(\bm\Phi_{t_j})$ is the smallest singular value of
$\bm\Phi_{t_j}$.
    
\end{theorem}
\noindent\textbf{Complexity.} For time complexity, the fitting in Eq.~\ref{eq:ridge-solve} is of $\gO(K(M+1)^2+K(M+1)F+(M+1)^3+(M+1)^2F)$, where $K$ is the number of cached timesteps, $M$ is the degree of Chebyshev basis, and $F$ is the latent feature dimension. Empirically, we have $M << K << F$, making $\gO(K(M+1)F)$ the dominant term. Compared with the Taylor forecaster, which is of complexity $\gO(PF)$ where $P$ is the Taylor expansion order,~\method only incurs marginal compute overhead, which is only of magnitude $K(M+1)/P$. The forecast operation in Eq.~\ref{eq:pred} is merely $\gO((M+1)F)$. Practically, both $K$ and $M$ are particularly small, making the overhead negligible, especially compared with the extremely computationally intensive operations in the actual denoiser forward passes. Regarding the memory, we only need to additionally preserve the coefficients matrix $\rmC$ besides the features in the cache $\sC$, which leads to a total of $\gO((M+1)F+KF)$ memory occupation.

\begin{table*}[t!]
  \centering
  \setlength{\tabcolsep}{3pt}
  \caption{Benchmark results of \textbf{text-to-image generation} task on DrawBench with Flux and Stable Diffusion 3.5-Large. We use 50 steps as the reference ($\dagger$). Our~\method achieves higher speedup while maintaining better sample quality across two speedup scenarios consistently.}
  \vskip -0.05in
  \newcommand{\splitheader}[1]{\begin{tabular}{@{}c@{}}#1\end{tabular}}
  \resizebox{\linewidth}{!}{
    \begin{tabular}{lcccccccccccccc}
    \toprule
          & \multicolumn{7}{c}{\textbf{FLUX.1}~\cite{flux2024}}                & \multicolumn{7}{c}{\textbf{Stable Diffusion 3.5-Large}~\cite{esser2024scaling}} \\
          \cmidrule(lr){2-8}\cmidrule(lr){9-15}
          & \multicolumn{2}{c}{Acceleration} & \multicolumn{3}{c}{Quality} & \multirow{2}[1]{*}{\splitheader{Image\\Reward$\uparrow$}} & \smash{\raisebox{-2ex}{CLIP$\uparrow$}} & \multicolumn{2}{c}{Acceleration} & \multicolumn{3}{c}{Quality} & \multirow{2}[1]{*}{\splitheader{Image\\Rewad$\uparrow$}} & \smash{\raisebox{-2ex}{CLIP$\uparrow$}} \\
           \cmidrule(lr){2-3}\cmidrule(lr){4-6} \cmidrule(lr){9-10}\cmidrule(lr){11-13}
          & Latency(s) $\downarrow$ & Speedup$\uparrow$ & PSNR$\uparrow$  & SSIM$\uparrow$  & LPIPS$\downarrow$ &       & & Latency(s)$\downarrow$ & Speedup$\uparrow$ & PSNR$\uparrow$  & SSIM$\uparrow$  & LPIPS$\downarrow$  &  \\
    \midrule
    50 steps$^\dagger$ &   26.03    &   1.00    &   -    &     -  &    -   &   1.00 &  27.49    &  25.05     &   1.00    &   -    &   -    &    -   &  1.05   &  28.66 \\
    
    25 steps &  13.23     &  1.97     &  18.11    &  0.752     &  0.322     &  1.01     &  27.58     
    
    &   12.67    &  1.98    &   12.41    &   0.593    &    0.464   & 1.02     &  28.78  \\

    15 steps &  8.04     &  3.24     &  15.77    &  0.673     &  0.432     &  1.00     &  27.56     
    
    &   7.71    &  3.25    &   10.51    &   0.453    &    0.595   &     0.85     &  28.75  \\
    
    \midrule
    
    FORA $(\mathcal{N}=4)$~\cite{selvaraju2024fora}  &   8.40    &  3.19     &  15.15     &  0.651     &  0.454     &  0.97     &  27.55     
    
    &  8.63     &  2.90     &   9.54    &  0.437     &   0.606    &  0.27     &  27.03 \\
    
    ToCa $(\mathcal{N}=4, \mathcal{R}=0.9)$~\cite{zou2024accelerating}  &   16.93    &  1.58     &  17.18     &  0.724     &  0.362     &  \textbf{1.03}     &  27.63     
    
    &   17.27    &  1.45    &  10.55   &  0.483     &   0.557    &  0.51    &  27.19 \\
           TeaCache $(\delta=0.4)$~\cite{liu2024timestep} &  10.66     &  2.44    &   18.70    &  0.762     &   0.311    &  \textbf{1.03}     &  \textbf{27.68}     
    &   10.75   &  2.33     &   11.89  &   0.527    &   0.512   &   0.78   &  27.46 \\
    
    TaylorSeer $(\mathcal{N}=4, \mathcal{O}=1)$~\cite{liu2025taylor} &   8.55    &  3.13     &  22.31     &  0.841     &  \textbf{0.215}     &  0.99     &  27.60     
    &   8.75    &  2.86     &   13.13    &  0.605     &   0.428    &  0.82     &  27.91 \\
    
    TaylorSeer $(\mathcal{N}=4, \mathcal{O}=2)$~\cite{liu2025taylor} &   8.59    &  3.03     &  20.76     &  0.812     &  0.247     &  1.02     &  27.61     
    &   8.70    &  2.88     &   11.12    &  0.513     &   0.529    &   0.44    &  27.24 \\
    \cellcolor{gray!15}{\method ($\alpha=0.75$)} & \cellcolor{gray!15}{\textbf{7.73}}      &  \cellcolor{gray!15}{\textbf{3.47}}     &   \cellcolor{gray!15}{\textbf{24.32}}    &  \cellcolor{gray!15}{\textbf{0.854}}     &   \cellcolor{gray!15}{0.217}    &   \cellcolor{gray!15}{0.99}    &   \cellcolor{gray!15}{27.65}    
    
    &   \cellcolor{gray!15}{\textbf{7.82}}      &  \cellcolor{gray!15}{\textbf{3.21}}     &   \cellcolor{gray!15}{\textbf{17.83}}    &  \cellcolor{gray!15}{\textbf{0.743}}     &   \cellcolor{gray!15}{\textbf{0.305}}    &   \cellcolor{gray!15}{\textbf{0.97}}    &   \cellcolor{gray!15}{\textbf{28.61}} \\
    
    \midrule
    
    FORA $(\mathcal{N}=6)$~\cite{selvaraju2024fora}  &   6.49    &  4.13     &  14.33     &  0.613     &  0.505     &  0.85     &  27.30     
    
    &   6.73   &  3.72    &   9.05   &  0.389    &   0.687   &  0.02    &  26.19 \\
    
    ToCa $(\mathcal{N}=6, \mathcal{R}=0.9)$~\cite{zou2024accelerating}  &   13.38    &  1.95     &  16.49     &  0.666     &  0.439     &  0.99     &  27.61     
    
    &   13.99    &  1.79     &   9.53    &  0.446     &   0.611    &  0.17     &  26.33 \\
        TeaCache $(\delta=0.8)$~\cite{liu2024timestep} &  6.75     &  3.85     &   16.71    &  0.695     &   0.406    &  \textbf{1.00}     &  27.63     
    &   7.17   &  3.49     &   10.04  &   0.472    &   0.589   &   0.39   &  27.08 \\
    TaylorSeer $(\mathcal{N}=6, \mathcal{O}=1)$~\cite{liu2025taylor} &  6.48     &  4.14     &   20.24    &  0.785     &   0.294    &  \textbf{1.00}     &  27.60     
    &   6.78    &  3.69     &   11.18    &  0.505     &   0.535    &  0.32     &  27.27 \\
    TaylorSeer $(\mathcal{N}=6, \mathcal{O}=2)$~\cite{liu2025taylor} &  6.52     &  3.99     &   17.41    &  0.708     &   0.389    &  0.99     &  27.39     
    &   6.76    &  3.71     &   9.36    &  0.429     &   0.636    &   0.10    &  25.90 \\
    \cellcolor{gray!15}{\method $(\alpha=3.0)$} & \cellcolor{gray!15}{\textbf{5.60}}      &  \cellcolor{gray!15}{\textbf{4.79}}     &   \cellcolor{gray!15}{\textbf{22.21}}    &  \cellcolor{gray!15}{\textbf{0.788}}     &   \cellcolor{gray!15}{\textbf{0.261}}    &   \cellcolor{gray!15}{\textbf{1.00}}    &   \cellcolor{gray!15}{\textbf{27.64}}    
    
    &   \cellcolor{gray!15}{\textbf{5.81}}     &  \cellcolor{gray!15}{\textbf{4.32}}     &   \cellcolor{gray!15}{\textbf{15.68}}    &  \cellcolor{gray!15}{\textbf{0.620}}     &    \cellcolor{gray!15}{\textbf{0.430}}   &  \cellcolor{gray!15}{\textbf{0.82}}    &   \cellcolor{gray!15}{\textbf{28.00}} \\
    
    \bottomrule
    \end{tabular}%
    }
    \vskip -0.1in
  \label{tab:image-maintable}%
\end{table*}%

\noindent\textbf{Caching the final block.} The original Taylor forecaster~\cite{liu2025taylor} maintains a cache for \emph{each modules} in attention blocks, which introduces $L$ times more memory and compute overhead where $L$ is the number of blocks. In practice, we find the layer-wise design redundant, and we instead only instantiate~\method for the output of the final attention block. Through extensive experiments, we verify that this strategy consistently yields on par or even better sample quality compared with layer-wise caching, while significantly reducing compute and memory overhead.

\noindent\textbf{Adaptive scheduling.} Since the diffusion sampling procedure in Alg.~\ref{alg:ourmethod} is indeed performed iteratively from $t=0$ to 1, the approximation errors of the forecaster induced in early diffusion steps will accumulate in the later steps, significantly amplifying the final error. To this end, we also introduce an adaptive timestep scheduler that encourages more actual network passes at the early sampling stage and gradually increases the usage of the forecaster at later diffusion steps. Specifically, for a total of $N$ diffusion steps, we select $\sU=\{\tau_j: j=\floor{\alpha \frac{r(r+1)}{2}}, r\in\sN^+, 1\leq j\leq N\}$ where $\alpha\geq 0$ is the slope of the time interval between adjacent timesteps in $\sU$. More explanations are in the Appendix.

\section{Experiments}
\label{sec:exp}
In this section, we benchmark our spectral feature forecaster against existing cache-then-predict approaches on both text-to-image generation (Sec.~\ref{sec:exp-t2i}) and text-to-video generation (Sec.~\ref{sec:exp-t2v}). Furthermore, we ablate our key designs and verify the efficacy of each component in Sec.~\ref{sec:exp-ablation}. 

\begin{figure}[t!]
\vskip -0.1in
    \centering
    \includegraphics[width=0.99\linewidth]{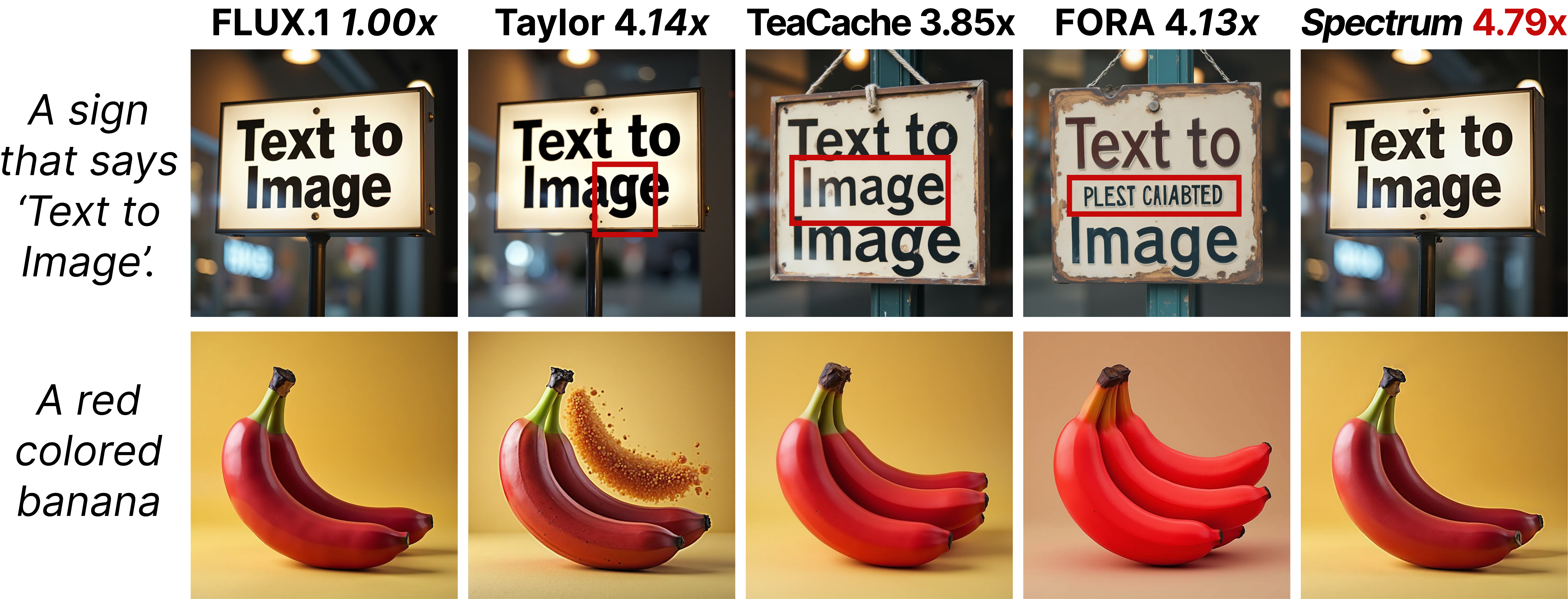}
    \vskip -0.1in
    \caption{\textbf{Qualitative comparison on text-to-image generation} using FLUX.1.~\method aligns consistently with the 50-step reference while accelerating it by a factor of 4.79$\times$. Other baselines show noticeable degradation in color and prompt consistency.}
    \label{fig:image-qualitative-compare}
    \vskip -0.25in
\end{figure}

\subsection{Accelerating Text-to-Image Diffusion Models}
\label{sec:exp-t2i}
\noindent\textbf{Experiment setup.} For text-to-image generation, we benchmark on two state-of-the-art image diffusion models, FLUX.1-dev~\cite{flux2024} and Stable Diffusion 3.5-Large~\cite{esser2024scaling}. We by default adopt 50 sampling steps as the reference, following established configurations. Experiments are performed on NVIDIA A100 GPUs for this task. We employ DrawBench~\cite{saharia2022photorealistic} as the prompt suite which contains 200 prompts, and we set the image resolution to 1024$\times$1024 for both models. We by default set $\lambda=0.1$ and $M=4$ for~\method with detailed ablations presented in Sec.~\ref{sec:exp-ablation}.

\noindent\textbf{Baselines.}  We compare~\method against the following baselines: \emph{$\kappa$ steps}, where we naively reduce the sampling step to $\kappa\leq 50$; \emph{FORA}~\cite{selvaraju2024fora} and \emph{ToCa}~\cite{zou2024accelerating}, which rely on direct cache reuse; \emph{TeaCache}, which employs a dynamic caching scheduling; and \emph{TaylorSeer}~\cite{liu2025taylor}, which predicts future cache values using Taylor expansion. We comprehensively benchmark the methods in two scenarios with different speedups controlled by their hyperparameters $\gN$ and $\alpha$. In particular, $\gN=4$ and $\gN=6$ correspond to 16 and 12 network evaluations, while $\alpha=0.75$ and $\alpha=3.0$ imply 14 and 10 network evaluations, respectively.

\noindent\textbf{Metrics.} We report the sampling latency (seconds per image) and the empirical speedup compared with the original 50-step sampler. For image quality, we first evaluate PSNR, SSIM, LPIPS between the generated images for each approach and the reference images generated by the 50-step sampler. We also include the ImageReward~\cite{xuImageRewardLearningEvaluating2023} and CLIP Score~\cite{hessel2021clipscore} of the generated images.

\begin{table*}[t!]
  \centering
\setlength{\tabcolsep}{5pt}
  \caption{Benchmark results of \textbf{text-to-video generation} task on VBench with Wan2.1-14B and HunyuanVideo. We use 50 steps as the reference ($\dagger$). Our~\method achieves higher speedup while maintaining better sample quality across two speedup scenarios consistently.}
  \vskip -0.1in
  \newcommand{\splitheader}[1]{\begin{tabular}{@{}c@{}}#1\end{tabular}}
  \resizebox{\linewidth}{!}{
    \begin{tabular}{lcccccccccccc}
    \toprule
          & \multicolumn{6}{c}{\textbf{Wan2.1-14B}~\cite{wan2025}}                & \multicolumn{6}{c}{\textbf{HunyuanVideo}~\cite{kong2024hunyuanvideo}} \\
          \cmidrule(lr){2-7}\cmidrule(lr){8-13}
          & \multicolumn{2}{c}{Acceleration} & \multicolumn{3}{c}{Quality} & \multirow{2}[1]{*}{\splitheader{VBench\\Quality$\uparrow$}} & \multicolumn{2}{c}{Acceleration} & \multicolumn{3}{c}{Quality} & \multirow{2}[1]{*}{\splitheader{VBench\\Quality$\uparrow$}} \\
           \cmidrule(lr){2-3}\cmidrule(lr){4-6} \cmidrule(lr){8-9}\cmidrule(lr){10-12}
          & Latency(s) $\downarrow$ & Speedup$\uparrow$ & PSNR$\uparrow$  & SSIM$\uparrow$  & LPIPS$\downarrow$ &       & Latency(s)$\downarrow$ & Speedup$\uparrow$ & PSNR$\uparrow$  & SSIM$\uparrow$  & LPIPS$\downarrow$  &  \\
    \midrule
    50 steps$^\dagger$ &   486.12    &   1.00    &   -    &     -  &    -   &   83.15    &  378.30     &   1.00    &   -    &   -    &    -   &  84.61 \\
    25 steps &  246.30     &  1.97     &  14.94    &  0.459     &  0.481     &  81.74     &   193.19    &  1.96    &   19.64    &   0.673    &    0.376   & 84.78  \\
    15 steps &   150.50    &  3.23     &  13.92     &   0.397    &  0.546     &  80.77     &   119.58    &  3.16     &   17.40    &  0.620     &   0.459    & 83.21 \\
    \midrule
    FORA $(\gN=4)$~\cite{selvaraju2024fora}  &   155.80    &   3.12    &  14.47     &  0.426     &   0.532    &  80.53     &   122.42    &  3.09     &  18.24     &   0.663    &    0.427   &  83.34 \\
    ToCa $(\gN=4,\gR=0.9)$~\cite{zou2024accelerating}  &   313.62    &   1.55    &   16.88    &    0.537   &  0.410     &  82.18     &  247.25     &   1.53    &   21.26    &  0.732     &  0.357     & 83.79 \\
    TeaCache~$(\delta=0.2)$~\cite{liu2024timestep} &   164.78    &  2.95     &  19.13     &   0.628   &  0.347     &   82.59    &   130.44    &  2.90     &   23.88   &  0.783    &   0.270    &  84.08 \\
        TaylorSeer~$(\gN=4,\gO=1)$~\cite{liu2025taylor} &   161.50    &  3.01     &  19.46     &  0.660     &  0.299     &  82.74     &   127.55    &  2.96     &   24.73    &  0.805     &   0.243    &  84.03 \\
    \cellcolor{gray!15}{\method $(\alpha=0.75)$} & \cellcolor{gray!15}{\textbf{143.03}}      &  \cellcolor{gray!15}{\textbf{3.40}}     &   \cellcolor{gray!15}{\textbf{22.78}}    &  \cellcolor{gray!15}{\textbf{0.749}}     &   \cellcolor{gray!15}{\textbf{0.222}}    &   \cellcolor{gray!15}{\textbf{82.80}}    &  \cellcolor{gray!15}{\textbf{112.59}}     &  \cellcolor{gray!15}{\textbf{3.36}}     &   \cellcolor{gray!15}{\textbf{27.77}}    &  \cellcolor{gray!15}{\textbf{0.842}}     &    \cellcolor{gray!15}{\textbf{0.209}}   &  \cellcolor{gray!15}{\textbf{84.11}} \\
    \midrule
    FORA $(\gN=6)$~\cite{selvaraju2024fora}  &   120.92    &  4.02     &   13.02    &  0.353     &  0.598     &   80.82    &   95.77    & 3.95 &  17.29     &  0.594     &  0.475     & 83.35 \\
    ToCa $(\gN=6,\gR=0.9)$~\cite{zou2024accelerating}  &   250.57    &   1.94    &   15.10    &  0.488     &   0.453    &   81.14    &   199.10    &   1.90    &   19.48    &  0.656     &  0.438     & 83.62 \\
    TeaCache~$(\delta=0.25)$~\cite{liu2024timestep} &  127.59    &  3.81 &    16.53   &  0.557   &  0.390     &   81.85    &   102.52   &  3.69    &   21.42   &  0.725    &   0.334    &  83.68 \\
        TaylorSeer~$(\gN=6,\gO=1)$~\cite{liu2025taylor} &  123.49     &  3.94     &   17.24    &  0.585     &   0.367    &  81.38     &  97.83     &  3.86     &   22.27    &   0.740    &   0.303    & 83.40  \\
    \cellcolor{gray!15}{\method $(\alpha=3.0)$} & \cellcolor{gray!15}{\textbf{104.18}}      &  \cellcolor{gray!15}{\textbf{4.67}}     &   \cellcolor{gray!15}{\textbf{21.24}}    &  \cellcolor{gray!15}{\textbf{0.694}}     &   \cellcolor{gray!15}{\textbf{0.265}}    &   \cellcolor{gray!15}{\textbf{82.21}}    &  \cellcolor{gray!15}{\textbf{82.90}}     &  \cellcolor{gray!15}{\textbf{4.56}}     &   \cellcolor{gray!15}{\textbf{25.39}}    &  \cellcolor{gray!15}{\textbf{0.779}}     &    \cellcolor{gray!15}{\textbf{0.273}}   &  \cellcolor{gray!15}{\textbf{83.89}} \\
    \bottomrule
    \end{tabular}%
    }
  \label{tab:video-maintable}%
  \vskip -0.1in
\end{table*}%

\begin{figure}[t!]
\vskip -0.1in
    \centering
    \includegraphics[width=0.95\linewidth]{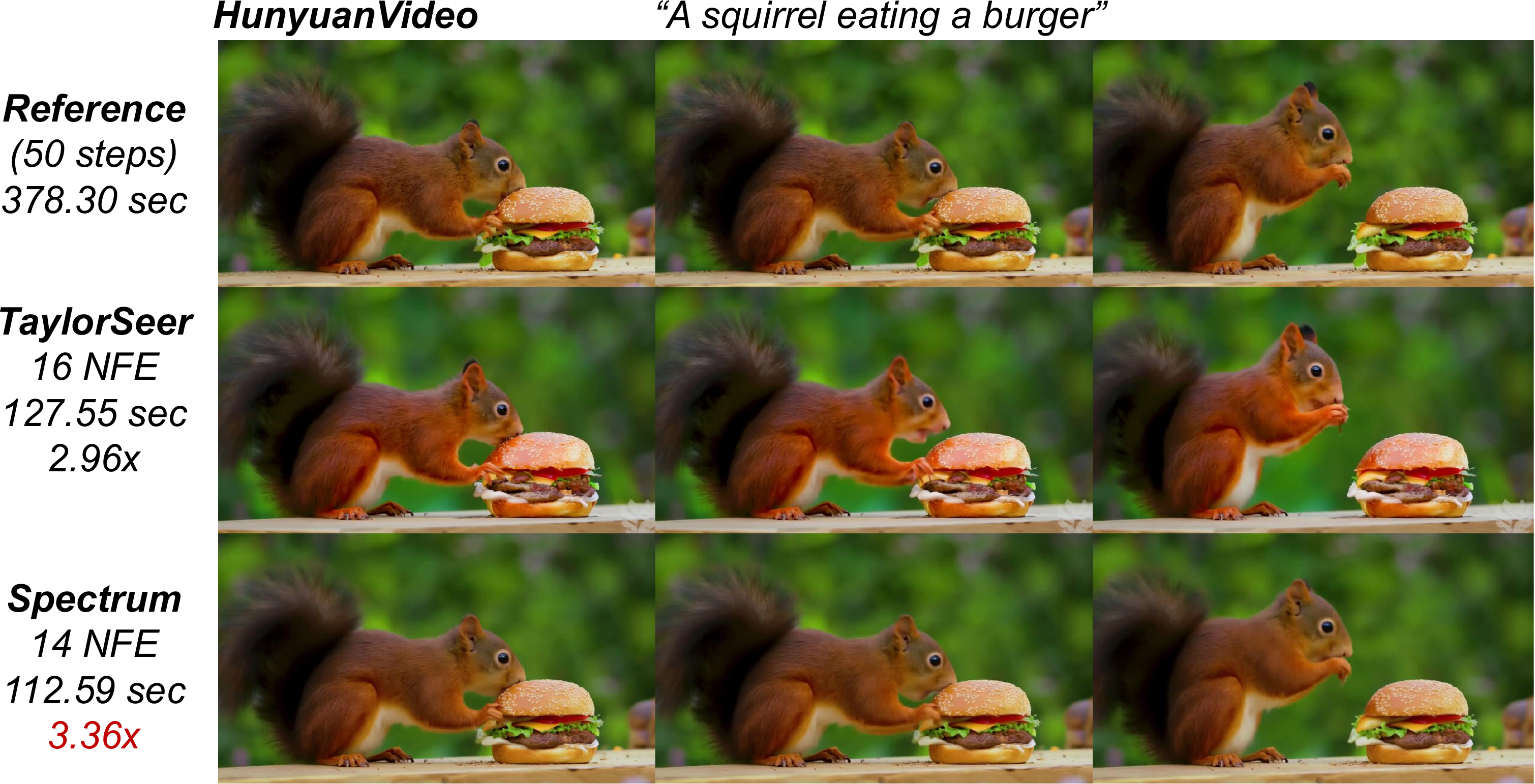}
    \vskip -0.1in
    \caption{\textbf{Qualitative comparisons} on HunyuanVideo.~\method achieves higher sample fidelity while delivering more speedup.}
    \vskip -0.25in
    \label{fig:hunyuan-compare}
\end{figure}

\begin{figure*}[t!]
    \centering
    \includegraphics[width=0.99\linewidth]{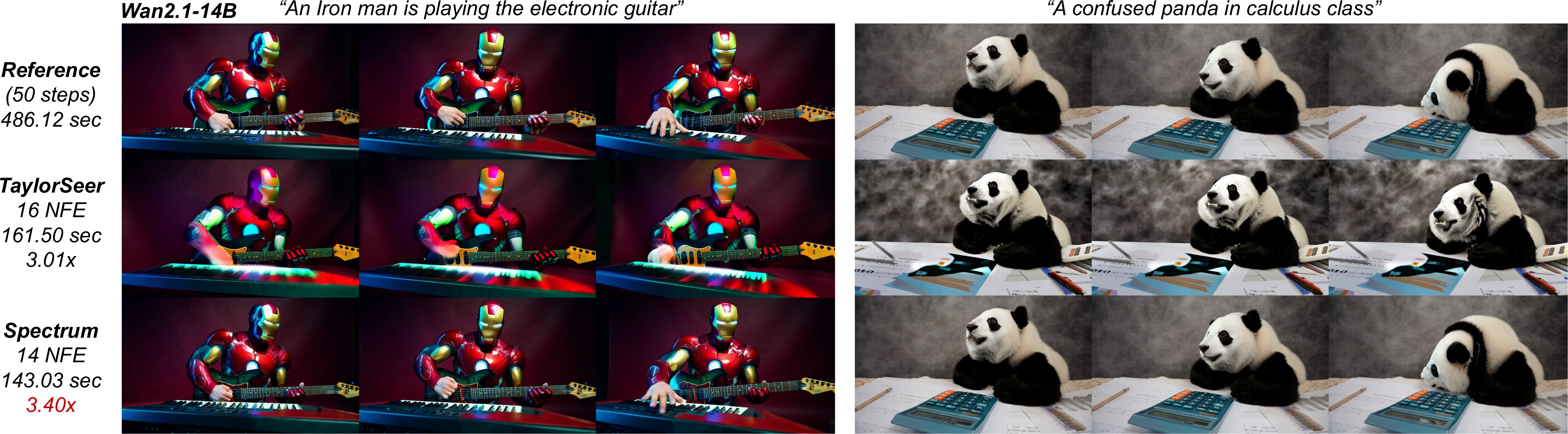}
    \vskip -0.05in
    \caption{\textbf{Qualitative comparison on text-to-video generation} using Wan2.1-14B.~\method aligns consistently with the high-quality 50-step reference using only 14 network evaluations, while TaylorSeer is slower and exhibits noticeable artifacts on character and background. }
    \vskip -0.25in
    \label{fig:video-qualitative-compare}
\end{figure*}

\noindent\textbf{Results.} The results and qualitative comparisons are in Table~\ref{tab:image-maintable} and Fig.~\ref{fig:image-qualitative-compare}. Notably,~\method consistently achieves higher speedup while maintaining much better sample quality, demonstrating the efficacy of our spectral feature forecaster. Furthermore, TaylorSeer encounters severe quality degradation in the higher speedup scenario ($\alpha=3.0$ or $\gN=6$) with remarkable drops in the quality metrics (\emph{e.g.}, 0.82 to 0.32 on ImageReward with SD3.5) due to the enlarged error in the approximation. Naively increasing the order of Taylorseer does not address the problem, which aligns with our analysis. In contrast, our~\method remains robust across all speedups, achieving much higher quality while using fewer network evaluations, thanks to the spectral forecasting approach with bounded error and better long-horizon behavior. Remarkably,~\method $(\alpha=3.0)$ achieves nearly the same PSNR with TaylorSeer $(\gN=4)$, while only using 10 network evaluations as opposed to 16.

\subsection{Accelerating Text-to-Video Diffusion Models}
\label{sec:exp-t2v}

% \looseness=-1
\noindent\textbf{Experiment setup.} We further evaluate ~\method on two advanced text-to-video diffusion models, Wan2.1-14B~\cite{wan2025} and HunyuanVideo~\cite{kong2024hunyuanvideo}. We adopt VBench~\cite{huang2023vbench} as the benchmark suite and directly leverage their prompt set and evaluation protocol. The inference for all approaches is conducted on NVIDIA H800 GPUs. Besides PSNR, SSIM, and LPIPS, we also report the VBench Quality Score (\%) as a quantitative measurement of video quality.

\noindent\textbf{Qualitative comparison.} We provide visualization of the generated videos on Hunyuan and Wan in Fig.~\ref{fig:hunyuan-compare} and~\ref{fig:video-qualitative-compare}, respectively (more in appendix). While TaylorSeer clearly exhibits unexpected artifacts on both the character and background,~\method consistently remains high-quality and aligns accurately with the reference 50-step samples.

\noindent\textbf{Results.} As depicted in Table~\ref{tab:video-maintable},~\method exhibits exceptional effectiveness on text-to-video generation, achieving up to 4.67$\times$ wallclock speedup with negligible effect on quality scores. While the quality drops even more significantly for the baselines, with TaylorSeer merely getting a PSNR of 17.24 on Wan2.1-14B, our~\method keeps producing high-fidelity video samples, reaching 21.24 PSNR with only 10 network passes. The superior quality obtained with extremely limited number of network evaluations again endorses our core design of modeling the latent features with spectral basis.

\subsection{Ablation Studies}
\label{sec:exp-ablation}

\noindent\textbf{The effect of regularization weight $\lambda$.} As depicted in Eq.~\ref{eq:ridge-solve}, the factor $\lambda$ serves as the regularization strength in the ridge regression problem. Here we investigate the empirical effect of $\lambda$ on FLUX.1 under different speedups by sweeping over $\alpha$, with results in Fig.~\ref{fig:ablate_lambda}. It is observed that both $\lambda=10^{-3}$ and $\lambda=10$ lead to detriment of image quality at each acceleration ratio, since a large $\lambda$ leads to underfitting of the features and hence larger forecasting error, while a small $\lambda$ adds insufficient diagonal regularization and incurs numerical instability when solving Eq.~\ref{eq:ridge-solve}.

\begin{figure}[t]  % [t] = top, [b] = bottom, [h] = here
  \centering
  \includegraphics[width=0.92\linewidth]{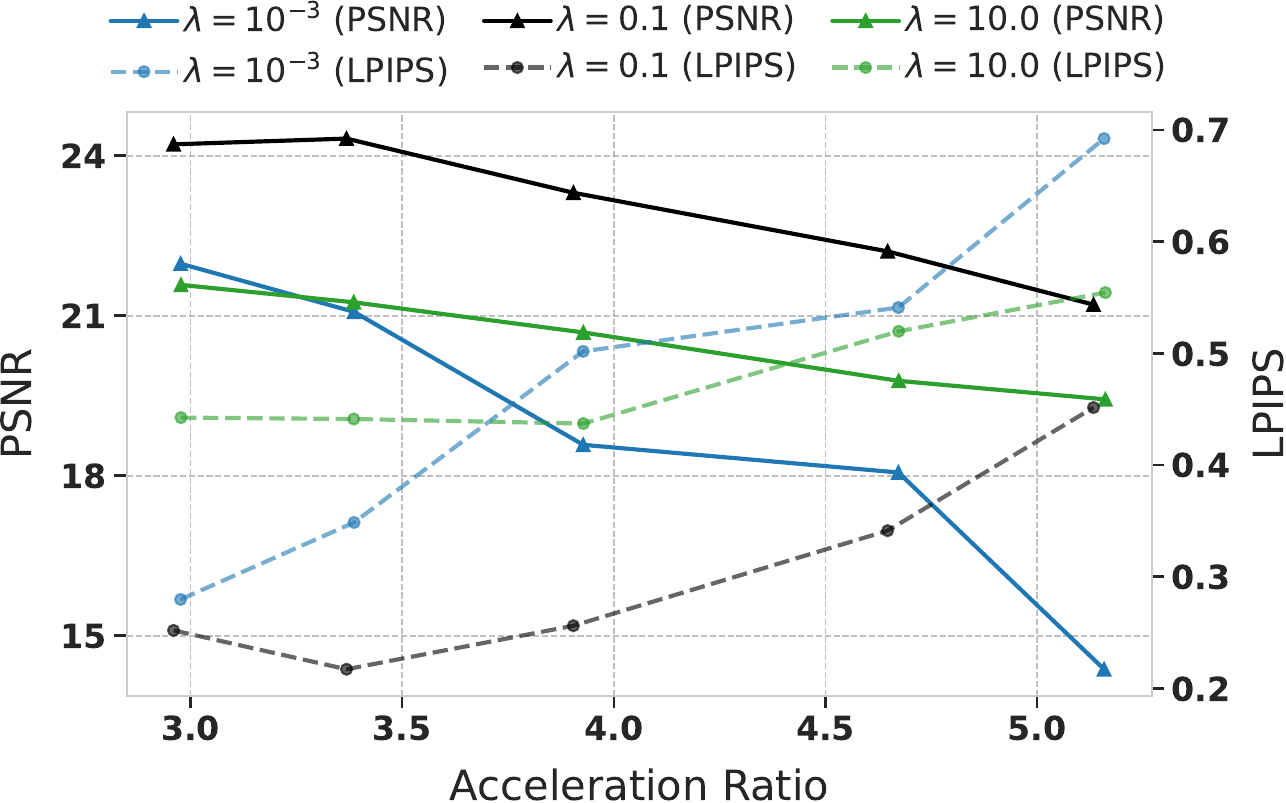}
  \vskip -0.1in
  \caption{Ablation study on the regularization weight $\lambda$.}
  \label{fig:ablate_lambda}
  \vskip -0.15in
\end{figure}

\noindent\textbf{The effect of adaptive scheduling.} We further study the efficacy of our proposed adaptive scheduling strategy, with results displayed in Table~\ref{tab:ablate_adaptive_scheduling}. Notably, using the same number of network evaluations, our adaptive scheduling $(\alpha=3.0)$ yields enhanced sample quality compared with fixed stepsize scheduling $(\gN=8)$ as evidenced by higher PSNR, SSIM, and Image Reward. The improvement remains consistent across both Taylor and our spectral forecaster, further validating our proposal. It is worth highlighting that~\method significantly promotes the sample quality over TaylorSeer regardless of the scheduling. For instance,~\method enhances PSNR from 17.24 to 20.65 on Wan2.1 even under the fixed scheduling $(\gN=8)$, which, again, verifies the efficacy and robustness of our spectral forecaster.

\begin{table}[t!]
  \centering
  \setlength{\tabcolsep}{4.5pt}
  \vskip -0.02in
  \caption{Ablation study of adaptive scheduling on SD3.5 and Wan.}
  \vskip -0.12in
  % \vskip -0.5in
    \resizebox{\linewidth}{!}{
    \begin{tabular}{lccccc}
    \toprule
          & \multicolumn{5}{c}{\textbf{Stable Diffusion 3.5-Large}~\cite{esser2024scaling}} \\
    % \midrule
    \cmidrule(lr){2-6}
    & Adaptive & PSNR$\uparrow$  & SSIM$\uparrow$  & LPIPS$\downarrow$ & Image Reward$\uparrow$ \\
    \midrule
    Taylor $(\gN=8)$ & $\times$    &   10.40    &   0.460    &  0.594     &  -0.22 \\
    Taylor $(\alpha=3.0)$ & $\checkmark$    &  13.25     &  0.585     &  0.454     & 0.65  \\
    \method $(\gN=8)$ & $\times$    &   11.34    &   0.501    &  0.552     &  0.23 \\
    \cellcolor{gray!15}{\method $(\alpha=3.0)$} & \cellcolor{gray!15}{$\checkmark$}     &   \cellcolor{gray!15}{\textbf{15.68}}    &  \cellcolor{gray!15}{\textbf{0.620}}     &   \cellcolor{gray!15}{\textbf{0.430}}    & \cellcolor{gray!15}{\textbf{0.82}} \\
    \midrule
          & \multicolumn{5}{c}{\textbf{Wan2.1-14B}~\cite{wan2025}} \\
    % \midrule
     \cmidrule(lr){2-6}
     & Adaptive & PSNR$\uparrow$  & SSIM$\uparrow$  & LPIPS$\downarrow$ & VBench Score$\uparrow$ \\
    \midrule
    Taylor $(\gN=8)$    & $\times$     &   17.24    &  0.584     &   0.367    & 81.38  \\
     Taylor $(\alpha=3.0)$ & $\checkmark$     &  19.85     &  0.674     &  0.283     & 81.67 \\
    \method $(\gN=8)$ & $\times$     &  20.65     &  0.687     &  0.275     &  \textbf{82.32}   \\
    \cellcolor{gray!15}{\method $(\alpha=3.0)$} & \cellcolor{gray!15}{$\checkmark$}     &   \cellcolor{gray!15}{\textbf{21.24}}    &   \cellcolor{gray!15}{\textbf{0.694}}    &  \cellcolor{gray!15}{\textbf{0.265}}     & \cellcolor{gray!15}{82.21} \\
    \bottomrule
    \end{tabular}%
    }
  \label{tab:ablate_adaptive_scheduling}%
  \vskip -0.22in
\end{table}%

\noindent\textbf{Ablation on the degree of Chebyshev polynomials $M$.} Per our analysis, the degree of Chebyshev polynomials, $M$, is a key factor: A small $M$ will lead to imprecise modeling of the function $h_i(t)$ and thus incurs more forecasting error, while a large $M$ may bring unnecessary computation and memory overhead. The empirical observations on FLUX.1 in Fig.~\ref{fig:ablate_m} well align with our analysis. The sample quality is promoted as $M$ increases from 2 to 4, while further scaling up $M$ to 6 brings marginal improvement. We therefore adopt the choice of $M=4$ which is sufficiently effective for modeling the per-channel scalar temporal function $h_i(t)$.

\begin{figure}[t]  % [t] = top, [b] = bottom, [h] = here
  \centering
  \includegraphics[width=0.92\linewidth]{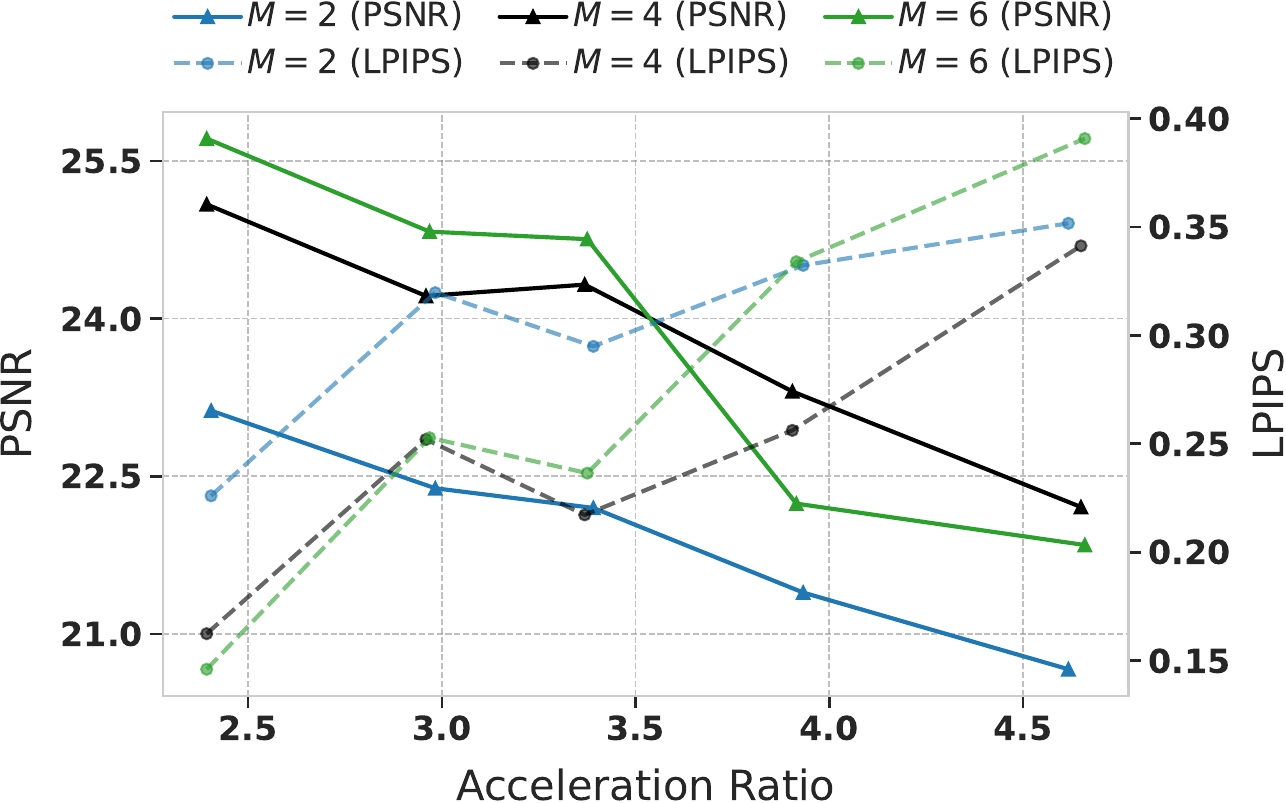}
  \vskip -0.1in
  \caption{Ablation on the degree of Chebyshev polynomials $M$.}
  \label{fig:ablate_m}
  \vskip -0.15in
\end{figure}

\begin{table}[t!]
  \centering
  \caption{Ablation study of the last-block-only caching strategy on FLUX.1 with $\gN=8$ for both Taylor and~\method.}
  \vskip -0.1in
      \resizebox{\linewidth}{!}{
    \begin{tabular}{lcccccc}
    \toprule
          & Last only & Latency(s)$\downarrow$ & PSNR$\uparrow$  & SSIM$\uparrow$  & LPIPS$\downarrow$ & Image Reward$\uparrow$ \\
    \midrule
    Taylor &   $\times$    &  7.36     &  17.93     &  0.718     &  0.375     & 0.99 \\
    Taylor &   $\checkmark$    &  5.45     &  18.61     &   0.738    &  0.356     &  0.99\\
    \method  &   $\times$    &  15.44     &   19.34    &   0.725    &   0.362    &  1.02 \\
    \cellcolor{gray!15}{\method}  &   \cellcolor{gray!15}{$\checkmark$}    & \cellcolor{gray!15}{5.59}      &  \cellcolor{gray!15}{19.66}     &   \cellcolor{gray!15}{0.741}    &   \cellcolor{gray!15}{0.341}  & \cellcolor{gray!15}{1.03}  \\
    \bottomrule
    \end{tabular}%
    }
  \label{tab:ablate_last_block}%
  \vskip -0.1in
\end{table}%

\begin{table}[t!]
\small
% \vskip -0.2in
  \centering
  \caption{RMSE between predicted latents and oracle on Wan2.1.}
  \vskip -0.1in
  \resizebox{0.99\linewidth}{!}{
    \begin{tabular}{lccccc}
    \toprule
           Diffusion step & 10    & 20    & 30    & 40    & 50 \\
    \midrule
     Taylor ($\gN=8$) & 0.0121 & 0.0303 & 0.0629 & 0.1226 & 0.2510 \\
    \cellcolor{gray!15}{\emph{Spectrum} ($\alpha=3.0$)} & \cellcolor{gray!15}{0.0040} &\cellcolor{gray!15}{0.0164} & \cellcolor{gray!15}{0.0358} & \cellcolor{gray!15}{0.0742} & \cellcolor{gray!15}{0.1674} \\
    \bottomrule
    \end{tabular}%
    }
    \vskip -0.2in
  \label{tab:latent_feature}%
\end{table}%

\noindent\textbf{The effect of direct last-block forecasting.} Empirically, we propose to only cache and forecast the output features from the last attention block. We validate such proposal in Table~\ref{tab:ablate_last_block} by controlling $\gN=8$. Specifically, compared with the original TaylorSeer, which performs caching and forecasting per module, operating only for the last block surprisingly yields even better performance while being much faster also more memory-efficient. We speculate that per-module forecasting incurs more frequent error accumulation when aggregating predicted features while the last-block-only modeling removes such redundancy.

\noindent\textbf{Feature-level analysis.} We evaluate the RMSE between the latents produced by the 50-step oracle and the forecasters at different diffusion steps in Table~\ref{tab:latent_feature}, which clearly shows that our spectral predictor produces diffusion features that are much closer to the oracle. More results are in Appendix.

\section{Conclusion}
\label{sec:conclusion}
We introduce~\method, a training-free diffusion sampling acceleration technique that caches and forecasts latent features in the denoiser. Our core idea is to  approximate latent features with Chebyshev polynomials, leading to a spectral forecaster with controlled error bound. Extensive experiments on image and video diffusion models demonstrate that \method can preserve high sample quality with much fewer network evaluations. \method takes a solid step toward high-throughput, real-time diffusion generation.

\clearpage

\noindent \textbf{Acknowledgments.}
We thank the anonymous reviewers for their feedback on improving the manuscript. This work was supported by ARO
(W911NF-21-1-0125), ONR (N00014-23-1-2159), the CZ Biohub, and Stanford HAI.

{
    \small
    \bibliographystyle{ieeenat_fullname}
    \bibliography{main}

@String(ICCV= {Int. Conf. Comput. Vis.})

@String(ICLR = {Int. Conf. Learn. Represent.})

@String(ICCV  = {ICCV})

@String(ICLR  = {ICLR})

@inproceedings{
    shih2023parallel,
    title={Parallel Sampling of Diffusion Models},
    author={Andy Shih and Suneel Belkhale and Stefano Ermon and Dorsa Sadigh and Nima Anari},
    booktitle={Thirty-seventh Conference on Neural Information Processing Systems},
    year={2023},
    url={https://openreview.net/forum?id=bpzwUfX1UP}
}

@article{ho2020denoising,
  title={Denoising diffusion probabilistic models},
  author={Ho, Jonathan and Jain, Ajay and Abbeel, Pieter},
  journal={Advances in neural information processing systems},
  volume={33},
  pages={6840--6851},
  year={2020}
}

@inproceedings{han2025chords,
  title={CHORDS: Diffusion Sampling Accelerator with Multi-core Hierarchical ODE Solvers},
  author={Han, Jiaqi and Ye, Haotian and Li, Puheng and Xu, Minkai and Zou, James and Ermon, Stefano},
  booktitle={Proceedings of the IEEE/CVF International Conference on Computer Vision},
  pages={19386--19395},
  year={2025}
}

@article{sabour2025align,
  title={Align Your Flow: Scaling Continuous-Time Flow Map Distillation},
  author={Sabour, Amirmojtaba and Fidler, Sanja and Kreis, Karsten},
  journal={arXiv preprint arXiv:2506.14603},
  year={2025}
}

@misc{rivlin1974chebyshev,
  title={The Chebyshev Polynomials, Pure and Applied Mathematics},
  author={Rivlin, Theodore J and others},
  year={1974},
  publisher={Wiley--Interscience, New York}
}

@article{salimans2024multistep,
  title={Multistep distillation of diffusion models via moment matching},
  author={Salimans, Tim and Mensink, Thomas and Heek, Jonathan and Hoogeboom, Emiel},
  journal={Advances in Neural Information Processing Systems},
  volume={37},
  pages={36046--36070},
  year={2024}
}

@article{geng2025mean,
  title={Mean flows for one-step generative modeling},
  author={Geng, Zhengyang and Deng, Mingyang and Bai, Xingjian and Kolter, J Zico and He, Kaiming},
  journal={arXiv preprint arXiv:2505.13447},
  year={2025}
}

@book{mason2002chebyshev,
  title={Chebyshev polynomials},
  author={Mason, John C and Handscomb, David C},
  year={2002},
  publisher={Chapman and Hall/CRC}
}

@inproceedings{hessel2021clipscore,
  title={{CLIPScore:} A Reference-free Evaluation Metric for Image Captioning},
  author={Hessel, Jack and Holtzman, Ari and Forbes, Maxwell and Bras, Ronan Le and Choi, Yejin},
  booktitle={EMNLP},
  year={2021}
}

@InProceedings{huang2023vbench,
     title={{VBench}: Comprehensive Benchmark Suite for Video Generative Models},
     author={Huang, Ziqi and He, Yinan and Yu, Jiashuo and Zhang, Fan and Si, Chenyang and Jiang, Yuming and Zhang, Yuanhan and Wu, Tianxing and Jin, Qingyang and Chanpaisit, Nattapol and Wang, Yaohui and Chen, Xinyuan and Wang, Limin and Lin, Dahua and Qiao, Yu and Liu, Ziwei},
     booktitle={Proceedings of the IEEE/CVF Conference on Computer Vision and Pattern Recognition},
     year={2024}
 }

@inproceedings{
liu2023flow,
title={Flow Straight and Fast: Learning to Generate and Transfer Data with Rectified Flow},
author={Xingchao Liu and Chengyue Gong and qiang liu},
booktitle={The Eleventh International Conference on Learning Representations },
year={2023},
url={https://openreview.net/forum?id=XVjTT1nw5z}
}

@inproceedings{
lipman2023flow,
title={Flow Matching for Generative Modeling},
author={Yaron Lipman and Ricky T. Q. Chen and Heli Ben-Hamu and Maximilian Nickel and Matthew Le},
booktitle={The Eleventh International Conference on Learning Representations },
year={2023},
url={https://openreview.net/forum?id=PqvMRDCJT9t}
}

@article{kong2024hunyuanvideo,
  title={Hunyuanvideo: A systematic framework for large video generative models},
  author={Kong, Weijie and Tian, Qi and Zhang, Zijian and Min, Rox and Dai, Zuozhuo and Zhou, Jin and Xiong, Jiangfeng and Li, Xin and Wu, Bo and Zhang, Jianwei and others},
  journal={arXiv preprint arXiv:2412.03603},
  year={2024}
}

@inproceedings{
  song2021scorebased,
  title={Score-Based Generative Modeling through Stochastic Differential Equations},
  author={Yang Song and Jascha Sohl-Dickstein and Diederik P Kingma and Abhishek Kumar and Stefano Ermon and Ben Poole},
  booktitle={International Conference on Learning Representations},
  year={2021},
  url={https://openreview.net/forum?id=PxTIG12RRHS}
}

@article{saharia2022photorealistic,
  title={Photorealistic text-to-image diffusion models with deep language understanding},
  author={Saharia, Chitwan and Chan, William and Saxena, Saurabh and Li, Lala and Whang, Jay and Denton, Emily L and Ghasemipour, Kamyar and Gontijo Lopes, Raphael and Karagol Ayan, Burcu and Salimans, Tim and others},
  journal={Advances in neural information processing systems},
  volume={35},
  pages={36479--36494},
  year={2022}
}

@article{song2019generative,
  title={Generative modeling by estimating gradients of the data distribution},
  author={Song, Yang and Ermon, Stefano},
  journal={Advances in neural information processing systems},
  volume={32},
  year={2019}
}

@misc{peebles2023scalablediffusionmodelstransformers,
      title={Scalable Diffusion Models with Transformers}, 
      author={William Peebles and Saining Xie},
      year={2023},
      eprint={2212.09748},
      archivePrefix={arXiv},
      primaryClass={cs.CV},
      url={https://arxiv.org/abs/2212.09748}, 
}

@inproceedings{rombach2022high,
  title={High-resolution image synthesis with latent diffusion models},
  author={Rombach, Robin and Blattmann, Andreas and Lorenz, Dominik and Esser, Patrick and Ommer, Bj{\"o}rn},
  booktitle={Proceedings of the IEEE/CVF conference on computer vision and pattern recognition},
  pages={10684--10695},
  year={2022}
}

@article{song2020denoising,
  title={Denoising diffusion implicit models},
  author={Song, Jiaming and Meng, Chenlin and Ermon, Stefano},
  journal={arXiv preprint arXiv:2010.02502},
  year={2020}
}

@article{karras2022elucidating,
  title={Elucidating the design space of diffusion-based generative models},
  author={Karras, Tero and Aittala, Miika and Aila, Timo and Laine, Samuli},
  journal={Advances in Neural Information Processing Systems},
  volume={35},
  pages={26565--26577},
  year={2022}
}

@inproceedings{song2023consistency,
  title={Consistency Models},
  author={Yang Song and Prafulla Dhariwal and Mark Chen and Ilya Sutskever},
  booktitle={International Conference on Machine Learning},
  year={2023},
}

@inproceedings{li2023q,
  title={Q-diffusion: Quantizing diffusion models},
  author={Li, Xiuyu and Liu, Yijiang and Lian, Long and Yang, Huanrui and Dong, Zhen and Kang, Daniel and Zhang, Shanghang and Keutzer, Kurt},
  booktitle={Proceedings of the IEEE/CVF International Conference on Computer Vision},
  pages={17535--17545},
  year={2023}
}

@article{kim2023consistency,
  title={Consistency trajectory models: Learning probability flow ode trajectory of diffusion},
  author={Kim, Dongjun and Lai, Chieh-Hsin and Liao, Wei-Hsiang and Murata, Naoki and Takida, Yuhta and Uesaka, Toshimitsu and He, Yutong and Mitsufuji, Yuki and Ermon, Stefano},
  journal={arXiv preprint arXiv:2310.02279},
  year={2023}
}

@article{guan2025forecasting,
  title={Forecasting when to forecast: Accelerating diffusion models with confidence-gated taylor},
  author={Guan, Xiaoliu and Jiang, Lielin and Chen, Hanqi and Zhang, Xu and Yan, Jiaxing and Wang, Guanzhong and Liu, Yi and Zhang, Zetao and Wu, Yu},
  journal={Knowledge-Based Systems},
  pages={114635},
  year={2025},
  publisher={Elsevier}
}

@inproceedings{
lu2025simplifying,
title={Simplifying, Stabilizing and Scaling Continuous-time Consistency Models},
author={Cheng Lu and Yang Song},
booktitle={The Thirteenth International Conference on Learning Representations},
year={2025},
url={https://openreview.net/forum?id=LyJi5ugyJx}
}

@inproceedings{
song2024improved,
title={Improved Techniques for Training Consistency Models},
author={Yang Song and Prafulla Dhariwal},
booktitle={The Twelfth International Conference on Learning Representations},
year={2024},
url={https://openreview.net/forum?id=WNzy9bRDvG}
}

@inproceedings{
yin2024improved,
title={Improved Distribution Matching Distillation for Fast Image Synthesis},
author={Tianwei Yin and Micha{\"e}l Gharbi and Taesung Park and Richard Zhang and Eli Shechtman and Fredo Durand and William T. Freeman},
booktitle={The Thirty-eighth Annual Conference on Neural Information Processing Systems},
year={2024},
url={https://openreview.net/forum?id=tQukGCDaNT}
}

@inproceedings{esser2024scaling,
  title={Scaling rectified flow transformers for high-resolution image synthesis},
  author={Esser, Patrick and Kulal, Sumith and Blattmann, Andreas and Entezari, Rahim and M{\"u}ller, Jonas and Saini, Harry and Levi, Yam and Lorenz, Dominik and Sauer, Axel and Boesel, Frederic and others},
  booktitle={Forty-first international conference on machine learning},
  year={2024}
}

@inproceedings{tang2024acceleratingparallelsamplingdiffusion,
  title={Accelerating parallel sampling of diffusion models},
  author={Tang, Zhiwei and Tang, Jiasheng and Luo, Hao and Wang, Fan and Chang, Tsung-Hui},
  booktitle={Forty-first International Conference on Machine Learning},
  year={2024}
}

@article{yang2023diffusionm,
  title={Diffusion models: A comprehensive survey of methods and applications},
  author={Yang, Ling and Zhang, Zhilong and Song, Yang and Hong, Shenda and Xu, Runsheng and Zhao, Yue and Zhang, Wentao and Cui, Bin and Yang, Ming-Hsuan},
  journal={ACM computing surveys},
  volume={56},
  number={4},
  pages={1--39},
  year={2023},
  publisher={ACM New York, NY, USA}
}

@InProceedings{liu2025taylor,
    author    = {Liu, Jiacheng and Zou, Chang and Lyu, Yuanhuiyi and Chen, Junjie and Zhang, Linfeng},
    title     = {From Reusing to Forecasting: Accelerating Diffusion Models with TaylorSeers},
    booktitle = {Proceedings of the IEEE/CVF International Conference on Computer Vision (ICCV)},
    month     = {October},
    year      = {2025},
    pages     = {15853-15863}
}

@article{salimans2022progressive,
  title        = {Progressive distillation for fast sampling of diffusion models},
  author       = {Salimans, Tim and Ho, Jonathan},
  year         = 2022,
  journal      = {arXiv preprint arXiv:2202.00512}
}

@article{selvaraju2024fora,
  title        = {FORA: Fast-Forward Caching in Diffusion Transformer Acceleration},
  author       = {Selvaraju, Pratheba and Ding, Tianyu and Chen, Tianyi and Zharkov, Ilya and Liang, Luming},
  year         = 2024,
  journal      = {arXiv preprint arXiv:2407.01425}
}

@article{chen2024delta-dit,
  title        = {$\Delta$-DiT: A Training-Free Acceleration Method Tailored for Diffusion Transformers},
  author       = {Chen, Pengtao and Shen, Mingzhu and Ye, Peng and Cao, Jianjian and Tu, Chongjun and Bouganis, Christos-Savvas and Zhao, Yiren and Chen, Tao},
  year         = 2024,
  journal      = {arXiv preprint arXiv:2406.01125}
}

@inproceedings{sohl2015deep,
  title        = {Deep unsupervised learning using nonequilibrium thermodynamics},
  author       = {Sohl-Dickstein, Jascha and Weiss, Eric and Maheswaranathan, Niru and Ganguli, Surya},
  year         = 2015,
  booktitle    = {International conference on machine learning},
  pages        = {2256--2265},
  organization = {PMLR}
}

@inproceedings{zou2024accelerating, 
  title        = {Accelerating Diffusion Transformers with Token-wise Feature Caching},
  author       = {Zou, Chang and Liu, Xuyang and Liu, Ting and Huang, Siteng and Zhang, Linfeng},
  booktitle    = {Proceedings of the 13th International Conference on Learning Representations (ICLR 2025)},
  year         = {2025},
  url          = {https://openreview.net/forum?id=yYZbZGo4ei},  
  publisher    = {ICLR}
}

@inproceedings{Liu2025SpeCa,
    author = {Liu, Jiacheng and Zou, Chang and Lyu, Yuanhuiyi and Ren, Fei and Wang, Shaobo and Li, Kaixin and Zhang, Linfeng},
    title = {SpeCa: Accelerating Diffusion Transformers with Speculative Feature Caching},
    year = {2025},
    isbn = {9798400720352},
    publisher = {Association for Computing Machinery},
    address = {New York, NY, USA},
    url = {https://doi.org/10.1145/3746027.3755331},
    doi = {10.1145/3746027.3755331},
    booktitle = {Proceedings of the 33rd ACM International Conference on Multimedia},
    pages = {10024–10033},
    numpages = {10},
    location = {Dublin, Ireland},
    series = {MM '25}
}

@inproceedings{Zheng2025Compute,
  title        = {{Compute only 16 tokens in one timestep: Accelerating Diffusion Transformers with Cluster-Driven Feature Caching}},
  author       = {Zheng, Zhixin and Wang, Xinyu and Zou, Chang and Wang, Shaobo and Zhang, Linfeng},
  booktitle    = {Proceedings of the 33rd ACM International Conference on Multimedia (MM '25)},
  year         = {2025},
  pages        = {to appear},
  publisher    = {ACM},
  address      = {Dublin, Ireland},
  month        = {October},
  institution  = {
    Shanghai Jiao Tong University and 
    University of Electronic Science and Technology of China and 
    Shandong University
  }
}

@article{lu2022dpm,
  title        = {Dpm-solver: A fast ode solver for diffusion probabilistic model sampling in around 10 steps},
  author       = {Lu, Cheng and Zhou, Yuhao and Bao, Fan and Chen, Jianfei and Li, Chongxuan and Zhu, Jun},
  year         = 2022,
  journal      = {Advances in Neural Information Processing Systems},
  volume       = 35,
  pages        = {5775--5787}
}

@article{lu2022dpm++,
  title        = {Dpm-solver++: Fast solver for guided sampling of diffusion probabilistic models},
  author       = {Lu, Cheng and Zhou, Yuhao and Bao, Fan and Chen, Jianfei and Li, Chongxuan and Zhu, Jun},
  year         = 2022,
  journal      = {arXiv preprint arXiv:2211.01095}
}

@inproceedings{
zheng2023dpmsolvervF,
title={{DPM}-Solver-v3: Improved Diffusion {ODE} Solver with Empirical Model Statistics},
author={Kaiwen Zheng and Cheng Lu and Jianfei Chen and Jun Zhu},
booktitle={Thirty-seventh Conference on Neural Information Processing Systems},
year={2023},
url={https://openreview.net/forum?id=9fWKExmKa0}
}

@article{Dao2022FlashAttentionFA,
  title        = {FlashAttention: Fast and Memory-Efficient Exact Attention with IO-Awareness},
  author       = {Tri Dao and Daniel Y. Fu and Stefano Ermon and Atri Rudra and Christopher R'e},
  year         = 2022,
  journal      = {ArXiv},
  volume       = {abs/2205.14135},
  url          = {https://api.semanticscholar.org/CorpusID:249151871}
}

@inproceedings{shang2023post,
  title        = {Post-training quantization on diffusion models},
  author       = {Shang, Yuzhang and Yuan, Zhihang and Xie, Bin and Wu, Bingzhe and Yan, Yan},
  year         = 2023,
  booktitle    = {Proceedings of the IEEE/CVF conference on computer vision and pattern recognition},
  pages        = {1972--1981}
}

@inproceedings{yin2024stepdistillation,
  title        = {One-step diffusion with distribution matching distillation},
  author       = {Yin, Tianwei and Gharbi, Micha{\"e}l and Zhang, Richard and Shechtman, Eli and Durand, Fredo and Freeman, William T and Park, Taesung},
  year         = 2024,
  booktitle    = {Proceedings of the IEEE/CVF Conference on Computer Vision and Pattern Recognition},
  pages        = {6613--6623}
}

@inproceedings{lin2014coco,
  title        = {Microsoft coco: Common objects in context},
  author       = {Lin, Tsung-Yi and Maire, Michael and Belongie, Serge and Hays, James and Perona, Pietro and Ramanan, Deva and Doll{\'a}r, Piotr and Zitnick, C Lawrence},
  year         = 2014,
  booktitle    = {Computer Vision--ECCV 2014: 13th European Conference, Zurich, Switzerland, September 6-12, 2014, Proceedings, Part V 13},
  pages        = {740--755},
  organization = {Springer}
}

@misc{flux2024,
  title        = {FLUX},
  author       = {Black Forest Labs},
  year         = 2024,
  howpublished = {\url{https://github.com/black-forest-labs/flux}}
}

@misc{zhang2024tokenpruningcachingbetter,
      title={Token Pruning for Caching Better: 9 Times Acceleration on Stable Diffusion for Free}, 
      author={Evelyn Zhang and Bang Xiao and Jiayi Tang and Qianli Ma and Chang Zou and Xuefei Ning and Xuming Hu and Linfeng Zhang},
      year={2024},
      eprint={2501.00375},
      archivePrefix={arXiv},
      primaryClass={cs.CV},
      url={https://arxiv.org/abs/2501.00375}, 
}

@misc{zou2024DuCa,
      title={Accelerating Diffusion Transformers with Dual Feature Caching}, 
      author={Chang Zou and Evelyn Zhang and Runlin Guo and Haohang Xu and Conghui He and Xuming Hu and Linfeng Zhang},
      year={2024},
      eprint={2412.18911},
      archivePrefix={arXiv},
      primaryClass={cs.LG},
      url={https://arxiv.org/abs/2412.18911}, 
}

@misc{liu2025regionadaptivesamplingdiffusiontransformers,
      title={Region-Adaptive Sampling for Diffusion Transformers}, 
      author={Ziming Liu and Yifan Yang and Chengruidong Zhang and Yiqi Zhang and Lili Qiu and Yang You and Yuqing Yang},
      year={2025},
      eprint={2502.10389},
      archivePrefix={arXiv},
      primaryClass={cs.CV},
      url={https://arxiv.org/abs/2502.10389}, 
}

@misc{liu2024timestep,
    title={Timestep Embedding Tells: It's Time to Cache for Video Diffusion Model},
    author={Feng Liu and Shiwei Zhang and Xiaofeng Wang and Yujie Wei and Haonan Qiu and Yuzhong Zhao and Yingya Zhang and Qixiang Ye and Fang Wan},
    year={2024},
    eprint={2411.19108},
    archivePrefix={arXiv},
    primaryClass={cs.CV}
}

@misc{cheng2025catpruningclusterawaretoken,
      title={CAT Pruning: Cluster-Aware Token Pruning For Text-to-Image Diffusion Models}, 
      author={Xinle Cheng and Zhuoming Chen and Zhihao Jia},
      year={2025},
      eprint={2502.00433},
      archivePrefix={arXiv},
      primaryClass={cs.CV},
      url={https://arxiv.org/abs/2502.00433}, 
}

@misc{saghatchian2025cached,
    title={Cached Adaptive Token Merging: Dynamic Token Reduction and Redundant Computation Elimination in Diffusion Model},
    author={Omid Saghatchian and Atiyeh Gh. Moghadam and Ahmad Nickabadi},
    year={2025},
    eprint={2501.00946},
    archivePrefix={arXiv},
    primaryClass={cs.CV}
}

@misc{sun2025unicpunifiedcachingpruning,
      title={UniCP: A Unified Caching and Pruning Framework for Efficient Video Generation}, 
      author={Wenzhang Sun and Qirui Hou and Donglin Di and Jiahui Yang and Yongjia Ma and Jianxun Cui},
      year={2025},
      eprint={2502.04393},
      archivePrefix={arXiv},
      primaryClass={cs.CV},
      url={https://arxiv.org/abs/2502.04393}, 
}

@inproceedings{kim2025ditto,
  author = {Sungbin Kim and Hyunwuk Lee and Wonho Cho and Mincheol Park and Won Woo Ro},
  title = {Ditto: Accelerating Diffusion Model via Temporal Value Similarity},
  booktitle = {Proceedings of the 2025 IEEE International Symposium on High-Performance Computer Architecture (HPCA)},
  year = {2025},
  publisher = {IEEE},
}

@inproceedings{
    zhu2024dipgo,
    title={DiP-{GO}: A Diffusion Pruner via Few-step Gradient Optimization},
    author={haoweiz and Dehua Tang and Ji Liu and Mingjie Lu and Jintu Zheng and Jinzhang Peng and Dong Li and Yu Wang and Fan Jiang and Lu Tian and Spandan Tiwari and Ashish Sirasao and Jun-Hai Yong and Bin Wang and Emad Barsoum},
    booktitle={The Thirty-eighth Annual Conference on Neural Information Processing Systems},
    year={2024},
    url={https://openreview.net/forum?id=sbsaRj475E}
}

@misc{xuImageRewardLearningEvaluating2023,
	title = {{ImageReward}: {Learning} and {Evaluating} {Human} {Preferences} for {Text}-to-{Image} {Generation}},
	shorttitle = {{ImageReward}},
	url = {http://arxiv.org/abs/2304.05977},
	doi = {10.48550/arXiv.2304.05977},
	urldate = {2025-03-05},
	publisher = {arXiv},
	author = {Xu, Jiazheng and Liu, Xiao and Wu, Yuchen and Tong, Yuxuan and Li, Qinkai and Ding, Ming and Tang, Jie and Dong, Yuxiao},
	month = dec,
	year = {2023},
	note = {arXiv:2304.05977 [cs]},

}

@article{structural_pruning_diffusion,
  title        = {Structural Pruning for Diffusion Models},
  author       = {Fang, Gongfan and Ma, Xinyin and Wang, Xinchao},
  year         = 2023,
  journal      = {arXiv preprint arXiv:2305.10924}
}

@article{wan2025,
      title={Wan: Open and Advanced Large-Scale Video Generative Models}, 
      author={Team Wan and Ang Wang and Baole Ai and Bin Wen and Chaojie Mao and Chen-Wei Xie and Di Chen and Feiwu Yu and Haiming Zhao and Jianxiao Yang and Jianyuan Zeng and Jiayu Wang and Jingfeng Zhang and Jingren Zhou and Jinkai Wang and Jixuan Chen and Kai Zhu and Kang Zhao and Keyu Yan and Lianghua Huang and Mengyang Feng and Ningyi Zhang and Pandeng Li and Pingyu Wu and Ruihang Chu and Ruili Feng and Shiwei Zhang and Siyang Sun and Tao Fang and Tianxing Wang and Tianyi Gui and Tingyu Weng and Tong Shen and Wei Lin and Wei Wang and Wei Wang and Wenmeng Zhou and Wente Wang and Wenting Shen and Wenyuan Yu and Xianzhong Shi and Xiaoming Huang and Xin Xu and Yan Kou and Yangyu Lv and Yifei Li and Yijing Liu and Yiming Wang and Yingya Zhang and Yitong Huang and Yong Li and You Wu and Yu Liu and Yulin Pan and Yun Zheng and Yuntao Hong and Yupeng Shi and Yutong Feng and Zeyinzi Jiang and Zhen Han and Zhi-Fan Wu and Ziyu Liu},
      journal = {arXiv preprint arXiv:2503.20314},
      year={2025}
}

@inproceedings{
    lvFasterCacheTrainingFreeVideo2025,
    title={FasterCache: Training-Free Video Diffusion Model Acceleration with High Quality},
    author={Zhengyao Lv and Chenyang Si and Junhao Song and Zhenyu Yang and Yu Qiao and Ziwei Liu and Kwan-Yee K. Wong},
    booktitle={The Thirteenth International Conference on Learning Representations},
    year={2025},
    url={https://openreview.net/forum?id=W49UjcpGxx}
}
}


\appendix

\onecolumn

    \begin{center}
        \Large
        \textbf{\thetitle}\\
        \vspace{0.5em}Supplementary Material \\
        \vspace{1.0em}
        
    \end{center}

\section{Proofs}

\subsection{Proof of Theorem~\ref{thm:taylor-minimax-lb}}

\begin{proof}
The proof follows directly from Taylor's theorem with the Lagrange form of the remainder.
Let $f \in \mathcal{F}_{P+1}(L)$. For any $\tau_j, \tau_k \in [0, 1]$, Taylor's theorem states that there exists a $\xi$ between $\tau_k$ and $\tau_j$ such that:
\[
f(\tau_j) = \sum_{p=0}^P \frac{f^{(p)}(\tau_k)}{p!} (\tau_j - \tau_k)^p + \frac{f^{(P+1)}(\xi)}{(P+1)!} (\tau_j - \tau_k)^{P+1}.
\]
The predictor $T_P[f](\tau_j)$ is defined as the first term (the Taylor polynomial). Thus, the approximation error is the absolute value of the remainder term:
\[
|f(\tau_j) - T_P[f](\tau_j)| = \left| \frac{f^{(P+1)}(\xi)}{(P+1)!} (\tau_j - \tau_k)^{P+1} \right|.
\]
Since $f \in \mathcal{F}_{P+1}(L)$, we have $\|f^{(P+1)}\|_\infty \le L$, which implies $|f^{(P+1)}(\xi)| \le L$. Substituting the step size $\tau_j - \tau_k = (j-k)\delta_t$, we obtain the upper bound:
\[
|f(\tau_j) - T_P[f](\tau_j)| \le \frac{L}{(P+1)!} \left( (j-k)\delta_t \right)^{P+1}.
\]
To show that this is the supremum (worst-case error), consider the specific function $f^*(\tau) = \frac{L}{(P+1)!}(\tau - \tau_k)^{P+1}$. The $(P+1)$-th derivative of $f^*$ is identically $L$, so $f^* \in \mathcal{F}_{P+1}(L)$. For this function, the derivatives $f^{*(p)}(\tau_k)$ are 0 for all $p \le P$, meaning $T_P[f^*](\tau_j) = 0$. The error is exactly $|f^*(\tau_j)| = \frac{L}{(P+1)!}((j-k)\delta_t)^{P+1}$, which attains the bound.
\end{proof}

\subsection{Proof of Theorem~\ref{theo:universal}}

\begin{proof}
This result is a standard theorem in approximation theory, often referred to as Bernstein's Theorem for Chebyshev approximation.
Recall that any function $f$ analytic on the Bernstein ellipse $E_\rho$ can be expanded in a Chebyshev series $f(\tau) = \sum_{k=0}^\infty a_k T_k(\tau)$. The coefficients $a_k$ satisfy the geometric decay bound $|a_k| \le 2B \rho^{-k}$ for $k \ge 1$, where $B = \sup_{z \in E_\rho} |f(z)|$.

The truncation error $f(\tau) - p_M(\tau)$ consists of the tail of the series:
\[
f(\tau) - p_M(\tau) = \sum_{k=M+1}^\infty a_k T_k(\tau).
\]
Taking the infinity norm on $[-1, 1]$ and using the property that $|T_k(\tau)| \le 1$ for $\tau \in [-1, 1]$:
\[
\|f - p_M\|_\infty \le \sum_{k=M+1}^\infty |a_k| \le \sum_{k=M+1}^\infty 2B \rho^{-k}.
\]
The right-hand side is a geometric series with ratio $\rho^{-1} < 1$:
\[
\sum_{k=M+1}^\infty 2B \rho^{-k} = 2B \frac{\rho^{-(M+1)}}{1 - \rho^{-1}} = \frac{2B \rho^{-M}}{\rho - 1}.
\]
Thus, $\|f - p_M\|_\infty \le \frac{2B}{\rho - 1}\rho^{-M}$.
\end{proof}

\subsection{Proof of Theorem~\ref{theo:our-error}}
Let $h_i(t)$ be the true feature function and $p_M(t) = \bm\phi(\tau) \rmC^*$ be its optimal Chebyshev approximation of degree $M$, where $\rmC^*$ is the vector of ideal Chebyshev coefficients. From Theorem~\ref{theo:universal}, we define the truncation error function $e(t) = h_i(t) - p_M(t)$, bounded by $|e(t)| \le \epsilon_M$.

We observe data at cached timesteps $\sC_{t_j} = \{(\rvh_{t_k}, t_k)\}$. Let $\rmH \in \mathbb{R}^K$ be the vector of observed values for channel $i$ (dropping indices $t_j$ for simplicity). We can write the observation model as:
\[
\rmH = \bm\Phi \rmC^* + \rmE,
\]
where $\bm\Phi$ is the design matrix and $\rmE$ is the vector of truncation errors at the cached points, with $\|\rmE\|_\infty \le \epsilon_M$.
The fitted coefficients $\widehat{\rmC}$ are obtained via ridge regression:
\[
\widehat{\rmC} = (\bm\Phi^\top \bm\Phi + \lambda \rmI)^{-1} \bm\Phi^\top \rmH.
\]
Substituting $\rmH$:
\[
\widehat{\rmC} = (\bm\Phi^\top \bm\Phi + \lambda \rmI)^{-1} \bm\Phi^\top (\bm\Phi \rmC^* + \rmE).
\]
We want to bound the prediction error at a new time $t_j$ (with $\tau_j = g(t_j)$). The true value is $h_i(t_j) = \bm\phi(\tau_j) \rmC^* + e(t_j)$, and the prediction is $\widehat{h}_i(t_j) = \bm\phi(\tau_j) \widehat{\rmC}$. The error is:
\[
\widehat{h}_i(t_j) - h_i(t_j) = \bm\phi(\tau_j) (\widehat{\rmC} - \rmC^*) - e(t_j).
\]
Let $\rmA = \bm\Phi^\top \bm\Phi + \lambda \rmI$. Note that $\rmA^{-1} \bm\Phi^\top \bm\Phi = \rmI - \lambda \rmA^{-1}$.
The coefficient error is:
\[
\begin{aligned}
\widehat{\rmC} - \rmC^* &= \rmA^{-1} \bm\Phi^\top (\bm\Phi \rmC^* + \rmE) - \rmC^* \\
&= (\rmA^{-1} \bm\Phi^\top \bm\Phi - \rmI) \rmC^* + \rmA^{-1} \bm\Phi^\top \rmE \\
&= -\lambda \rmA^{-1} \rmC^* + \rmA^{-1} \bm\Phi^\top \rmE.
\end{aligned}
\]
Substituting this back into the prediction error:
\[
\widehat{h}_i(t_j) - h_i(t_j) = \bm\phi(\tau_j) \rmA^{-1} \bm\Phi^\top \rmE - \lambda \bm\phi(\tau_j) \rmA^{-1} \rmC^* - e(t_j).
\]
Using the triangle inequality:
\[
|\widehat{h}_i - h_i| \le |\bm\phi(\tau_j) \rmA^{-1} \bm\Phi^\top \rmE| + |\lambda \bm\phi(\tau_j) \rmA^{-1} \rmC^*| + |e(t_j)|.
\]
Now we bound each term.

For the first term, we have $|\bm\phi \rmA^{-1} \bm\Phi^\top \rmE| \le \|\bm\phi\|_2\|\rmA^{-1}\|_2 \|\bm\Phi^\top \rmE\|_2$. Since $|T_m| \le 1$, we have $\|\bm\phi(\tau_j)\|_2 \le \sqrt{M+1}$. We also have $\|\rmA^{-1}\|_2 \le \frac{1}{\sigma_{\min}^2(\bm\Phi) + \lambda}$, and $\|\bm\Phi^\top \rmE\|_2 = \|\sum_{k=1}^K \bm\phi(\tau_k)^\top E_k\|_2 \le \sum_{k=1}^K \|\bm\phi(\tau_k)\|_2 |E_k| \le K \sqrt{M+1} \epsilon_M$. Combining the above, we have $|\bm\phi \rmA^{-1} \bm\Phi^\top \rmE| \le \frac{(M+1)K}{\sigma_{\min}^2 + \lambda} \epsilon_M$.

For the second term, we have $|\lambda \bm\phi \rmA^{-1} \rmC^*| \le \lambda \|\bm\phi\|_2 \|\rmA^{-1}\|_2 \|\rmC^*\|_2$. By the coefficient bound, we get 
\[
     \|\rmC^*\|_2 \le \sqrt{\sum_{m=0}^\infty 4B^2 \rho^{-2m}} = 2B \sqrt{\frac{1}{1 - \rho^{-2}}}.
     \]
     Combining with the bound in the first term, we have $|\lambda \bm\phi \rmA^{-1} \rmC^*| \le \frac{\lambda \sqrt{M+1}}{\sigma_{\min}^2 + \lambda} \frac{2B}{\sqrt{1-\rho^{-2}}}$.

     For the third term, $|e(t_j)| \le \epsilon_M$.

     Therefore, combining the bound for the three terms above gives the final bound:
     \[
|h_i(t_j) - \widehat{h}_i(t_j)| \le \epsilon_M \left(1 + \frac{(M+1)K}{\sigma_{\min}^2 + \lambda}\right) + \frac{\lambda \sqrt{M+1}}{\sigma_{\min}^2 + \lambda} \frac{2B}{\sqrt{1 - \rho^{-2}}}.
\]

\section{More Method Details}
\subsection{Adaptive Scheduling}
Most previous works~\citep{selvaraju2024fora,zou2024accelerating,liu2025taylor} on diffusion caching employ a uniform scheduling in selecting $\sU$ \emph{i.e.}, the set of timesteps to perform full network evaluation. For this strategy, a full network evaluation is performed every fixed interval of length $\mathcal{N} \in \sN^+$. Under our framework, this corresponds to $\sU_{\text{uniform}} = \{1\} \cup \{\tau_j: j=r\mathcal{N}, r\in\sN^+, 1\leq j\leq N\}$, where $N=50$ is the total number of discretization timesteps. However, this scheduling does not account for the fact that errors made in early diffusion steps accumulate and disproportionately affect later steps, eventually degrading the final output more significantly.

Other works~\citep{liu2024timestep,lvFasterCacheTrainingFreeVideo2025,sun2025unicpunifiedcachingpruning,liu2025regionadaptivesamplingdiffusiontransformers} attempt to address this by adaptively adjusting the scheduling $\sU$, but their methods typically require tracking inference time metrics, which introduces additional computation overheads. In contrast, we propose a simple, precomputed adaptive scheduling that progressively increases the interval length between adjacent timesteps, formulated as:
\begin{align*}
    \sU=\{\tau_j: j=\floor{\alpha \frac{r(r+1)}{2}}, r\in\sN^+, 1\leq j\leq N\},
\end{align*}
where $\alpha\geq 0$ characterizes the \emph{rate of increase} of the interval $\tau_j-\tau_k$ between adjacent timesteps $j$ and $k$.

\noindent
\textbf{Unified parameterization.}
Besides this progressively increasing recomputation intervals, our actual implementation has two additional components: the initial interval size $\gN$ and the number of warm-up steps $\gW$, which are \emph{directly inherited} from the original setup in TaylorSeer~\cite{liu2025taylor}. The initial interval size $\gN$ determines the length of the first interval, from which subsequent intervals grow. The warm-up step $\gW$ corresponds to the length of the initial phase during which full network evaluations are computed at every step, before the caching and forecasting procedure begins. The detailed formulation of \method's adaptive scheduling incorporated with these two hyperparameters is as follows:
\begin{align*}
\sU
    &= \{\tau_j : j \in \sN^+,~ 1 \le j \le \gW \}~\cup~ \Big\{ \tau_j : 
        j = \gW 
        + \Big\lfloor (r+1)\gN +
            \alpha \frac{r(r+1)}{2}
        \Big\rfloor, r \in \sN,~ 1 \le j \le N \Big\}.
\end{align*}
Therefore, the uniform scheduling $\sU_{\text{uniform}}$ in previous works can be exactly recovered by setting $\alpha=0.0$.

Under this unified parameterization, we provide detailed specifications of the baselines, which are parameterized by $\gN$ and $\gW$, and~\method, which is additionally parameterized by $\alpha$, in Table~\ref{tab:detailed-scheduling}. For clarity, we also include their corresponding number of network evaluations (NFE), which is closely related to their actual wall clock time reported in the main tables.

\section{More Experiment Details}

% Experiment settings

\subsection{Model Settings}
\noindent\textbf{Text-to-image.} We employ FLUX.1-dev\footnote{\texttt{black-forest-labs/FLUX.1-dev}} and Stable Diffusion 3.5-Large\footnote{\texttt{stabilityai/stable-diffusion-3.5-large}} for text-to-image generation. For the main experiments, we generate images with 1024$\times$1024 resolution. We apply the default guidance scale, which is 3.5 for FLUX and 7.0 for Stable Diffusion 3.5-Large.

\begin{table}[t!]
  \centering
  \small
  \caption{\textbf{Detailed specification on the scheduler and the total number of network evaluations (NFE) for all methods.} ``Reference'' refers to the tables in the main paper where the corresponding method was mentioned.}
      % \resizebox{\linewidth}{!}{
    \begin{tabular}{cccccc}
    \toprule
      & Reference & $\gN$    & $\gW$ & $\alpha$     &  NFE \\
    \midrule
    FORA $(\mathcal{N}=4)$      &  Table~\ref{tab:image-maintable},~\ref{tab:video-maintable}      &   4    &  1     &    0.0   & 13 \\
     ToCa $(\gN=4)$     &    Table~\ref{tab:image-maintable},~\ref{tab:video-maintable}   &    4   &     3  &    0.0   &14  \\
     TaylorSeer $(\gN=4)$     &    Table~\ref{tab:image-maintable},~\ref{tab:video-maintable}   &    4   &  5     &   0.0    & 16 \\
       \method $(\alpha=0.75)$  &    Table~\ref{tab:image-maintable},~\ref{tab:video-maintable}   &   2    &   5    &  0.75     & 14 \\
       \midrule
           FORA $(\mathcal{N}=6)$      &  Table~\ref{tab:image-maintable},~\ref{tab:video-maintable}     &    6   &    1   &  0.0     & 9 \\
     ToCa $(\gN=6)$     &    Table~\ref{tab:image-maintable},~\ref{tab:video-maintable}   &   6    &    3   &    0.0   & 10 \\
      TaylorSeer $(\gN=6)$     &    Table~\ref{tab:image-maintable},~\ref{tab:video-maintable}   &   6    &   5    &   0.0    & 12 \\
        \method $(\alpha=3.0)$  &   Table~\ref{tab:image-maintable},~\ref{tab:video-maintable}    &    2   &   5    &   3.0    & 10 \\
        \midrule
      TaylorSeer $(\gN=8)$     &   Table~\ref{tab:ablate_adaptive_scheduling}    &     8  &    5   &     0.0  & 10 \\
        TaylorSeer $(\alpha=3.0)$     &  Table~\ref{tab:ablate_adaptive_scheduling}     &   2    &    5   &   3.0    & 10 \\
      \method $(\gN=8)$     &   Table~\ref{tab:ablate_adaptive_scheduling}    &   8    &   5    &      0.0 & 10 \\
        \method $(\alpha=3.0)$     &  Table~\ref{tab:ablate_adaptive_scheduling}     &    2   &   5    &    3.0   & 10 \\
    \bottomrule
    \end{tabular}%
    % }
  \label{tab:detailed-scheduling}%
\end{table}%

\noindent\textbf{Text-to-video.} We adopt Wan2.1-14B\footnote{\texttt{Wan-AI/Wan2.1-T2V-14B}} and HunyuanVideo\footnote{\texttt{hunyuanvideo-community/HunyuanVideo}} for text-to-video generation. We generate 480p videos with 81 frames following the prompt suite in VBench~\cite{huang2023vbench} and strictly follow the benchmark protocol of VBench to evaluate the quality metrics and VBench Quality Score. Following convention, we apply a guidance scale of 5.0 for Wan2.1-14B and 6.0 for HunyuanVideo throughout all experiments.

\begin{figure}[t!]
    \centering
    \includegraphics[width=0.99\linewidth]{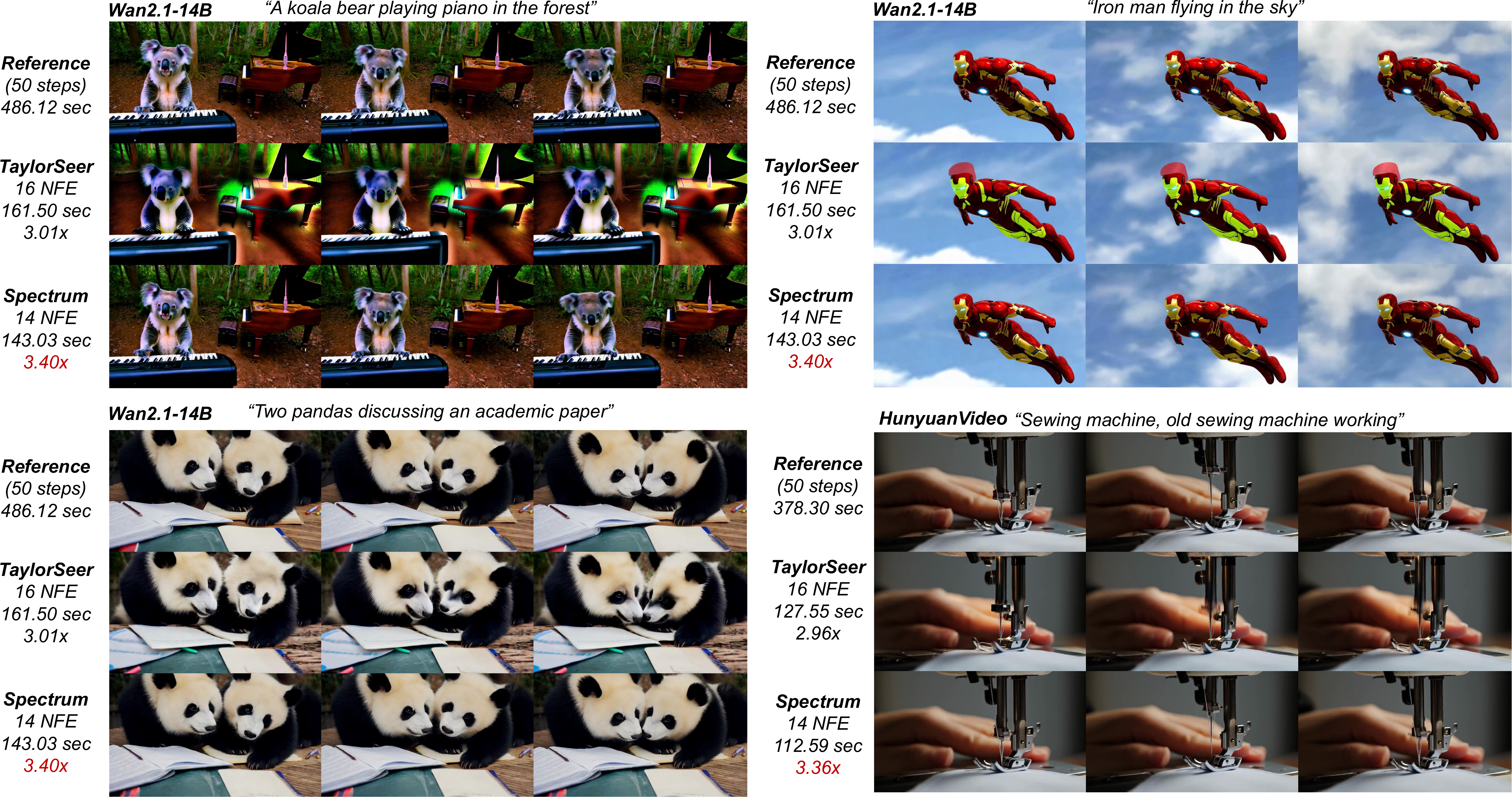}
    \caption{\textbf{Additional qualitative comparison on text-to-video generation} using Wan2.1-14B and HunyuanVideo.}
    \label{fig:more-text-to-video}
\end{figure}

\subsection{Baseline Settings}
For the baselines that employ a uniform activation scheduling (FORA, ToCa, and TaylorSeer), we set interval size $\mathcal{N}$ to be 4 in the slow acceleration setting and 6 in the high acceleration setting. We set $\gW=1$ for FORA, $R=90\%, \gW=3$ for ToCa, and $\gW=5$ for TaylorSeer to match the evaluation setting in~\citet{liu2025taylor}. TeaCache employs a non-deterministic scheduling that may vary across different runs, but we found that setting the caching threshold $\delta$ to $0.2$ and $0.8$ results in similar acceleration rates.

For \method, we use the same number of warm-up steps ($\gW=5$) as TaylorSeer and set $\gN=2$ for the settings that use the adaptive activation scheduling. We set $\alpha=0.75$ and $\alpha=3.0$, respectively, for the slow and high acceleration settings.

\subsection{Ablation Study Settings}
We conduct the ablation on the adaptive activation scheduling in the high-speed up setting to better visualize its impact. To ensure a fair comparison under the same computation budget, we set the Taylor polynomial method's interval size to $\gN=8$, which lets it consume the same number of network evaluations (NFE) as \method with $\alpha=3.0$. We conducted the ablation study on the effect of last-block forecasting at the preliminary stage of our project, so we used the uniform activation scheduling with $N=8$ for both the Taylor method and \method. We use adaptive scheduling for the ablation study on the regularization weight $\lambda$ and the degree of the Chebyshev polynomial $M$. The 5 different acceleration ratios in~\cref{fig:ablate_lambda,fig:ablate_m} correspond to the $\alpha$ values $[0.2, 0.4, 0.75, 1.5, 3.0]$, which incur the following NFEs $[20,17,14,12,10]$.

\section{More Results}
\begin{table*}[t!]
  \centering
  \setlength{\tabcolsep}{3pt}
  \caption{Benchmark results of \textbf{text-to-image generation} on COCO2017~\citep{lin2014coco} using FLUX.1. 
  We use 50 steps as the reference ($\dagger$). Our~\method achieves higher speedup while 
  maintaining better sample quality.}
  \vskip -0.05in
  \newcommand{\splitheader}[1]{\begin{tabular}{@{}c@{}}#1\end{tabular}}
  \resizebox{0.75\linewidth}{!}{
    \begin{tabular}{lccccccc}
    \toprule
          & \multicolumn{7}{c}{\textbf{FLUX.1}~\cite{flux2024}} \\
          \cmidrule(lr){2-8}
          & \multicolumn{2}{c}{Acceleration} 
          & \multicolumn{3}{c}{Quality} 
          & \multirow{2}[1]{*}{\splitheader{Image\\Reward$\uparrow$}} 
          & \smash{\raisebox{-2ex}{CLIP$\uparrow$}} 
          \\
           \cmidrule(lr){2-3}\cmidrule(lr){4-6}
          & Latency(s)$\downarrow$ & Speedup$\uparrow$ 
          & PSNR$\uparrow$  & SSIM$\uparrow$  & LPIPS$\downarrow$ 
          & & \\
    \midrule

    50 steps$^\dagger$ 
        & 26.46 & 1.00 & - & - & - & 1.13 & 25.92 \\

    \midrule

    FORA $(\mathcal{N}=6)$~\cite{selvaraju2024fora}  
        & 6.37 & 4.16 & 13.94 & 0.604 & 0.523 & 1.04 & \textbf{26.10} \\

    ToCa $(\mathcal{N}=6, \mathcal{R}=0.9)$~\cite{zou2024accelerating}  
        & 13.51 & 1.96 & 16.14 & 0.652 & 0.462 & 1.10 & \textbf{26.10} \\

    TeaCache $(\delta=0.8)$~\cite{liu2024timestep} 
        & 6.73 & 3.94 & 16.07 & 0.664 & 0.436 & 1.03 & 25.88 \\

    TaylorSeer $(\mathcal{N}=6, \mathcal{O}=1)$~\cite{liu2025taylor} 
        & 6.52 & 4.06 & 19.92 & 0.782 & 0.297 & 1.10 & 25.98 \\

    \cellcolor{gray!15}{\method $(\alpha=3.0)$} 
        & \cellcolor{gray!15}{\textbf{5.58}} 
        & \cellcolor{gray!15}{\textbf{4.75}} 
        & \cellcolor{gray!15}{\textbf{22.29}} 
        & \cellcolor{gray!15}{\textbf{0.811}} 
        & \cellcolor{gray!15}{\textbf{0.288}} 
        & \cellcolor{gray!15}{\textbf{1.11}} 
        & \cellcolor{gray!15}{25.95} \\

    \bottomrule
    \end{tabular}%
  }
  \vskip -0.1in
  \label{tab:image-maintable-coco}
\end{table*}
\noindent\textbf{Evaluation on COCO2017 captions dataset.} We provide additional results on the COCO2017 captions dataset~\citep{lin2014coco}  using FLUX.1 in Table~\ref{tab:image-maintable-coco}.~\method consistently exhibits exceptional performance on the new set of prompts.

\begin{figure*}[t!]
    \centering
    \includegraphics[width=0.80\linewidth]{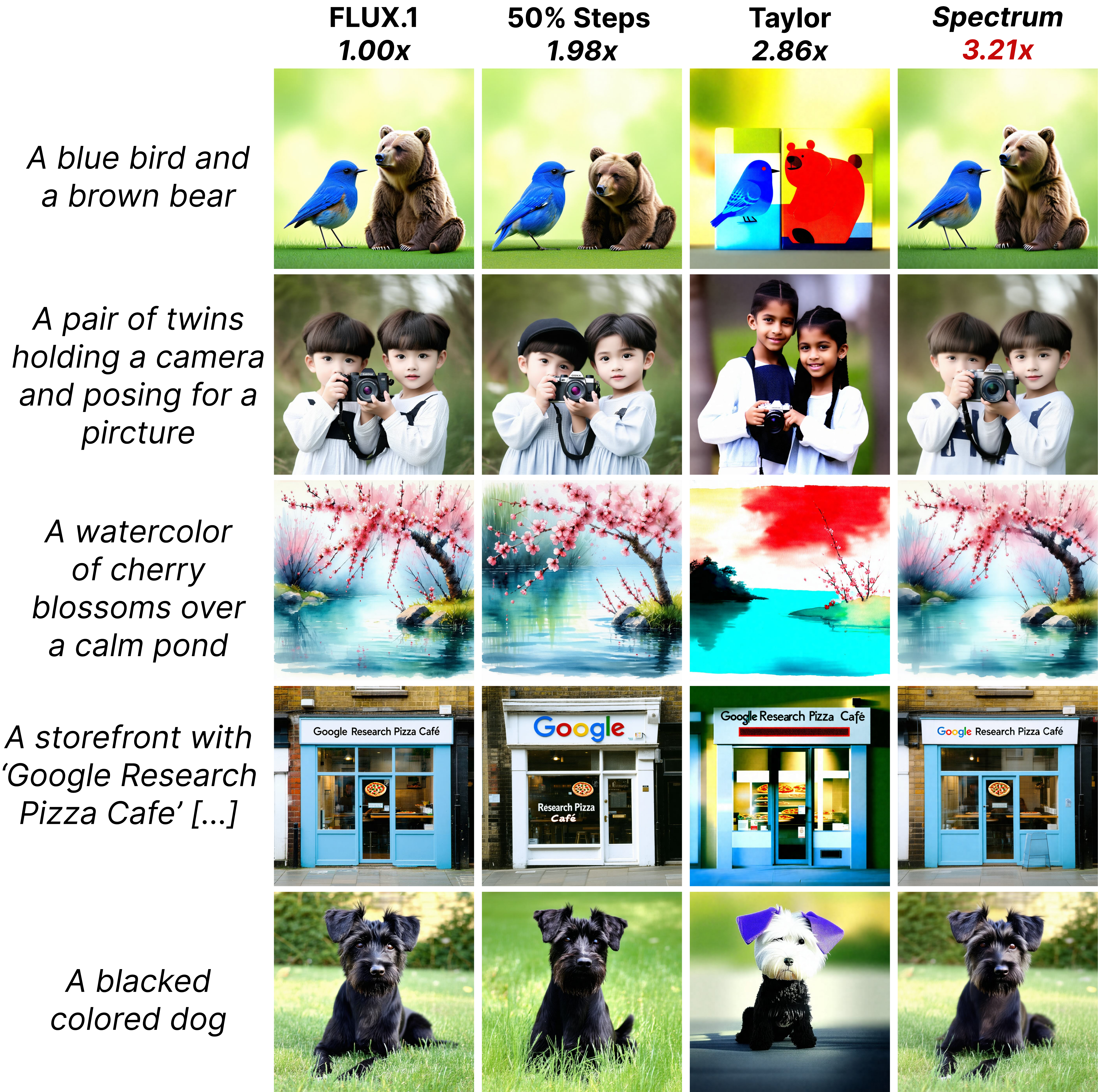}
    \caption{\textbf{Qualitative comparison on text-to-image generation} using Stable Diffusion 3.5-Large}
    \label{fig:image-qualitative-compare-sd3-5-full}
    % \vskip -0.2in
\end{figure*}

\noindent\textbf{More qualitative results.} We provide additional qualitative comparisons on text-to-image generation with FLUX.1 in Fig.~\ref{fig:image-qualitative-compare-flux-full} and with Stable Diffusion 3.5-Large in Fig.~\ref{fig:image-qualitative-compare-sd3-5-full}. We also provide more qualitative comparisons on text-to-image generation with Wan2.1-14B and HunyuanVideo in Fig.~\ref{fig:more-text-to-video}. We offer more visualizations of the generated video samples with~\method using HunyuanVideo in Fig.~\ref{fig:more-video-samples}.

\noindent\textbf{More latent feature results.} We provide RMSE on latent features on HunyuanVideo in Table~\ref{tab:latent_feature_hunyuan}.

\noindent\textbf{Generalization to U-Net architecture.} We additionally investigate the efficacy of \emph{Spectrum} on SDXL with U-Net architecture. Results are summarized in Fig.~\ref{fig:sdxl}, which shows that~\emph{Spectrum} remains highly effective on SDXL, outperforming TaylorSeer by a significant margin.

\begin{figure*}[t!]
    \centering
    \includegraphics[width=0.99\linewidth]{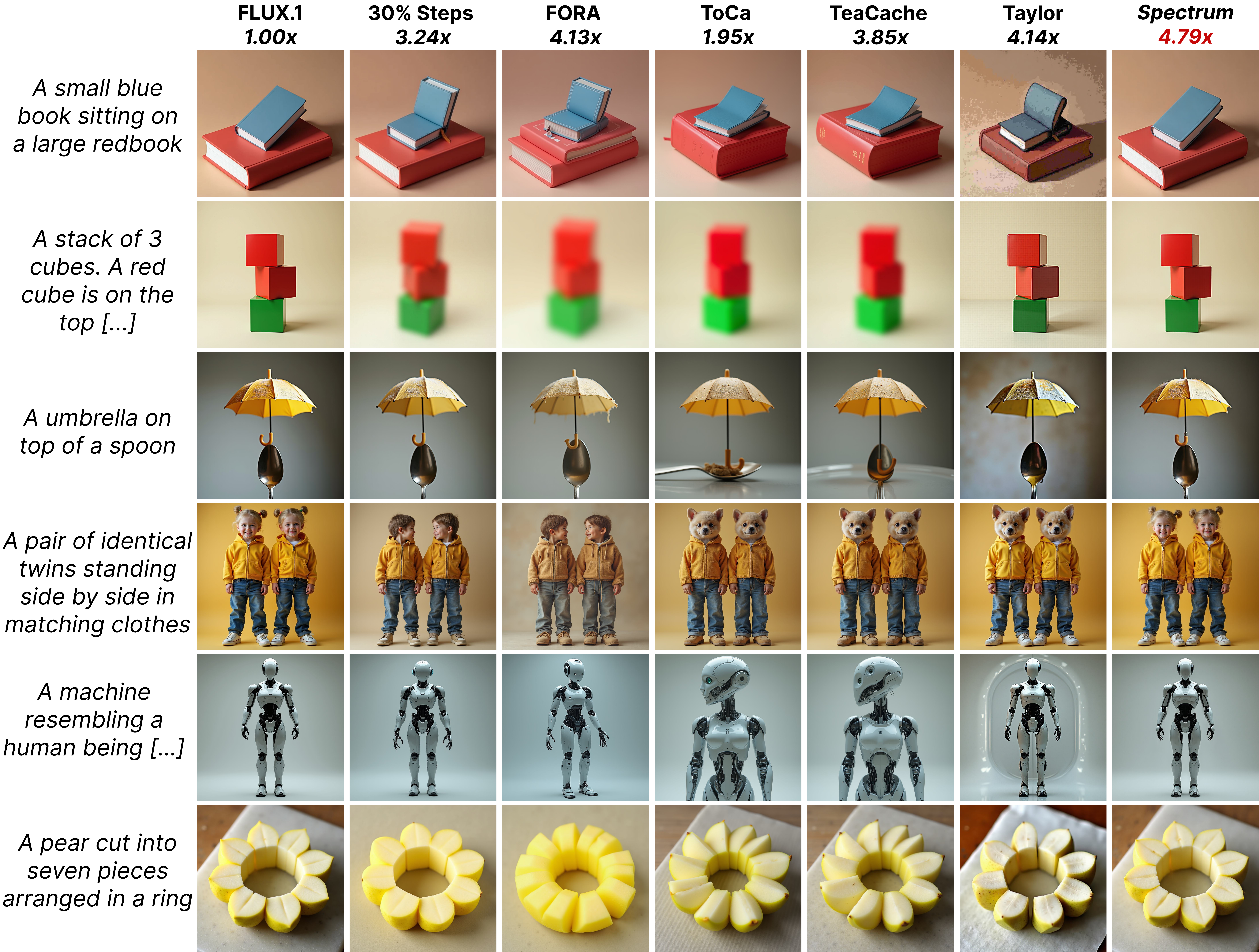}
    \caption{\textbf{Additional qualitative comparison on text-to-image generation} using FLUX.1.}
    \label{fig:image-qualitative-compare-flux-full}
    % \vskip -0.2in
\end{figure*}

% \noindent
\begin{figure}[t]
\centering

\begin{minipage}[t]{0.48\textwidth}
 \vspace{0pt}
  \centering
  \small
  \captionof{table}{RMSE between the predicted latent features and oracle.}
  \label{tab:latent_feature_hunyuan}
  \resizebox{0.88\linewidth}{!}{%
    \begin{tabular}{clccccc}
      \toprule
      & Diffusion step & 10 & 20 & 30 & 40 & 50 \\
      \midrule
      \multirow{2}{*}{Hunyuan} & TaylorSeer     & 0.0047 & 0.0102 & 0.0206 & 0.0467 & 0.1562 \\
                              & \emph{Spectrum} & 0.0021 & 0.0046 & 0.0081 & 0.0182 & 0.0818 \\
      \bottomrule
    \end{tabular}%
  }
\end{minipage}
\begin{minipage}[t]{0.48\textwidth}
 \vspace{0pt}
  \centering
  \includegraphics[width=0.9\linewidth]{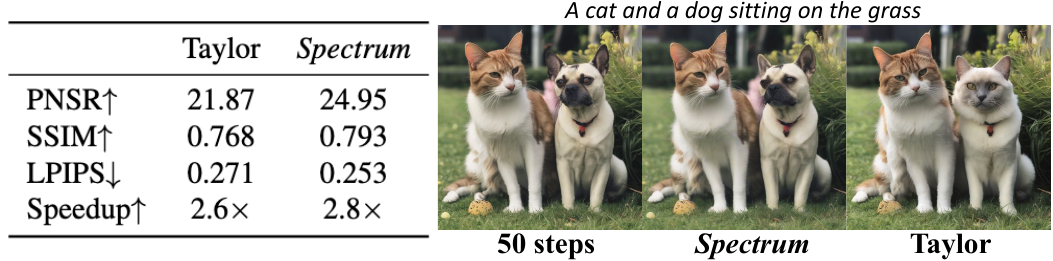}
  \vskip -0.1in
  \caption{Additional results on SDXL (U-Net architecture).}
  \label{fig:sdxl}
\end{minipage}

\end{figure}


\begin{figure*}[t!]
    \centering
    \includegraphics[width=0.99\linewidth]{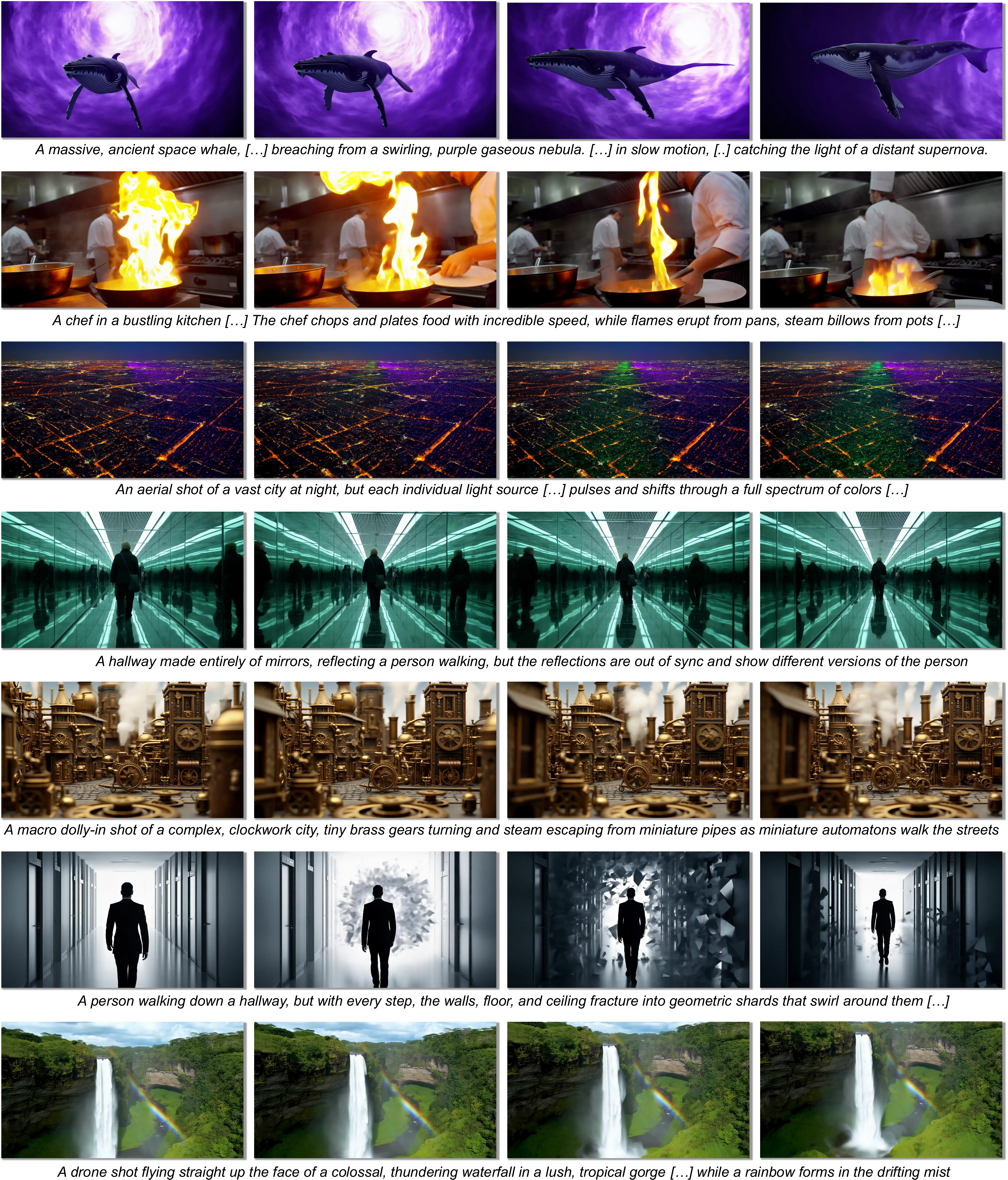}
    \caption{More \textbf{text-to-video generation samples} on HunyuanVideo with~\method. Samples were generated using only \textbf{14 network evaluations}, leading to a significant speedup of 3.5$\times$ without quality degradation.}
    \label{fig:more-video-samples}
\end{figure*}

\end{document}